\documentclass[letterpaper]{article} % DO NOT CHANGE THIS
\usepackage{aaai2026}  % DO NOT CHANGE THIS
\usepackage{times}  % DO NOT CHANGE THIS
\usepackage{helvet}  % DO NOT CHANGE THIS
\usepackage{courier}  % DO NOT CHANGE THIS
\usepackage[hyphens]{url}  % DO NOT CHANGE THIS
\usepackage{graphicx} % DO NOT CHANGE THIS
\usepackage[utf8]{inputenc}
\DeclareUnicodeCharacter{2009}{\,}
\usepackage{natbib}  % DO NOT CHANGE THIS AND DO NOT ADD ANY OPTIONS TO IT
\usepackage{caption} % DO NOT CHANGE THIS AND DO NOT ADD ANY OPTIONS TO IT
\usepackage{algorithm}
\usepackage{algorithmic}
\usepackage{times}
\usepackage{helvet}
\usepackage{courier}
\usepackage[table]{xcolor}
\usepackage{booktabs}
\usepackage{amsfonts}
\usepackage{amsmath}
\usepackage{amssymb} % for additional math symbols (e.g., \triangleq)
\usepackage{amsthm}
\theoremstyle{plain}      % default: bold heading, italic body
\newtheorem{theorem}{Theorem}   % reset counter each section
\newtheorem{corollary}[theorem]{Corollary}
\newtheorem{proposition}[theorem]{Proposition}
\newtheorem{lemma}[theorem]{Lemma}
\usepackage{newfloat}
\usepackage{bibunits} % allow separate bibliographies for main and appendix
\usepackage{listings}
\DeclareCaptionStyle{ruled}{labelfont=normalfont,labelsep=colon,strut=off} % DO NOT CHANGE THIS
\floatstyle{ruled}
\newfloat{listing}{tb}{lst}{}
\floatname{listing}{Listing}
\title{FRoD: Full-Rank Efficient Fine-Tuning with Rotational Degrees for Fast Convergence}
\author{
    Guoan Wan\textsuperscript{\rm 1},
    Tianyu Chen\textsuperscript{\rm 1},
    Fangzheng Feng\textsuperscript{\rm 2},
    Haoyi Zhou\textsuperscript{\rm 1},
    Runhua Xu\textsuperscript{\rm 1}\thanks{Corresponding author}
}
\affiliations{
    \textsuperscript{1}School of Computer Science and Engineering, Beihang University, China \\
    \textsuperscript{2}School of Electronic Information and Communications, Huazhong University of Science and Technoledgy, China \\
    \{gawan, tianyuc, haoyi, runhua\}@buaa.edu.cn, fangzhengfeng@hust.edu.cn
}

\usepackage{bibentry}
\begin{document}

\maketitle
\begin{bibunit}[aaai2026]

\begin{abstract}
Parameter-efficient fine-tuning (PEFT) methods have emerged as a practical solution for adapting large foundation models to downstream tasks, reducing computational and memory costs by updating only a small subset of parameters. Among them, approaches like LoRA aim to strike a balance between efficiency and expressiveness, but often suffer from slow convergence and limited adaptation capacity due to their inherent low-rank constraints. This trade-off hampers the ability of PEFT methods to capture complex patterns needed for diverse tasks. 
To address these challenges, we propose FRoD, a novel fine-tuning method that combines hierarchical joint decomposition with rotational degrees of freedom. 
By extracting a globally shared basis across layers and injecting sparse, learnable perturbations into scaling factors for flexible full-rank updates, FRoD enhances expressiveness and efficiency, leading to faster and more robust convergence.
On 20 benchmarks spanning vision, reasoning, and language understanding, FRoD matches full model fine-tuning in accuracy, while using only 1.72\% of trainable parameters under identical training budgets.
\end{abstract}

% Uncomment the following to link to your code, datasets, an extended version or similar.
% You must keep this block between (not within) the abstract and the main body of the paper.
\begin{links}
    \link{Code}{https://github.com/Bane-Elvin/AAAI2026-FRoD}
    % \link{Datasets}{https://aaai.org/example/datasets}
    % \link{Extended version}{https://aaai.org/example/extended-version}
\end{links}

\section{Introduction}
%% 第一句话建议删除，直接受第二句话 large foundation models have become xxx
% Large foundation models \cite{liu_roberta_2019} \cite{radford_learning_2021} have become the foundation of modern AI systems,  offering rich representations that benefit a wide range of downstream tasks in vision and language. As these models grow in size and capability, adapting them to specific downstream tasks via fine-tuning becomes increasingly important, but also increasingly expensive in terms of computation and memory. Parameter-efficient fine-tuning (PEFT) methods address this challenge by updating only a small subset of model parameters, significantly reducing resource demands. However, most PEFT approaches suffer from limited expressiveness and slow convergence due to their reliance on low-rank or fixed subspace constraints.

% 8.2 update
Large foundation models~\cite{liu_roberta_2019,radford_learning_2021} have unlocked powerful representations across domains, but their ever-growing parameter scales bring drastic computational and memory costs for downstream adaptation. 
% Large foundation models \cite{liu_roberta_2019,radford_learning_2021} offer increasingly powerful representations but simultaneously introduce significantly more parameters. Fine-tuning these parameters is essential for optimal performance on downstream tasks; however, fully fine-tuning all parameters is both computationally expensive and memory-intensive. 
Parameter-efficient fine-tuning (PEFT) paradigms—spanning additive modules~\cite{hu_llm-adapters_2023}, selective tuning~\cite{guo_parameter-efficient_2021}, and reparameterization~\cite{aghajanyan_intrinsic_2020}—mitigate this overhead by updating only sparse subsets of weights.
% Parameter-efficient fine-tuning (PEFT) methods, such as additive methods \cite{hu_llm-adapters_2023}, selective tuning \cite{guo_parameter-efficient_2021}, and reparameterization techniques \cite{aghajanyan_intrinsic_2020}, address this challenge by updating only a small subset of model parameters. 
Among these, LoRA~\cite{hu_lora_2021} injects low-rank adapters to enable efficient fine-tuning at virtually no extra inference cost.
% Among these, LoRA \cite{hu_lora_2021} inserts low-rank matrices into existing weights, enabling efficient fine-tuning without additional inference overhead. 
Despite these advances, achieving faster convergence with fewer parameters remains an ongoing research challenge.
%% 最后一句话应该是正向说，一直以来的目标是更少的参数和更快的收敛速度
% (RX TODO - 反向说的话有什么顾虑吗？
% However, achieving rapid convergence and full model adaptation with minimal parameters is still unresolved.）
% end

% \setlength{\belowcaptionskip}{-10pt}
% \begin{figure}[t]
% \centering
% \includegraphics[width=\linewidth]{figs/architecture.pdf}
% \caption{The architectural comparison among LoRA, PiSSA, VeRA, and FRoD. Blue modules denote frozen parameters, while orange modules indicate trainable components. Each subfigure depicts the trainable components and update structure of each method.}
% \label{fig:architecture}
% \end{figure}

% \setlength{\belowcaptionskip}{-10pt}
\begin{figure}[t]
\centering
\includegraphics[width=\linewidth]{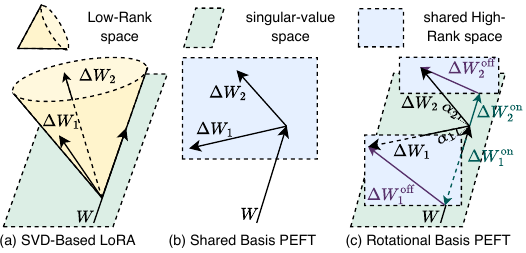}
\caption{
% Comparison with two existing approaches: 
% Our method uses joint matrix decomposition and off-axis sparse spaces, resulting in a higher rank update space than existing SVD-based LoRA and shared random-basis PEFT methods.
Update Space and Rank: FRoD (joint matrix decomposition and off-axis sparse spaces) vs. SVD-based and random PEFT methods.
}
\label{fig:intro}
\end{figure}

Prior approaches (\figurename~\ref{fig:intro}) like SVD-based LoRA~\cite{sun_singular_2022} use leading singular vectors for better initialization, 
yet storing full singular vector matrices~\cite{sun_svfit_2024, lingam_svft_2024} demands prohibitive memory, 
forcing most variants~\cite{meng_pissa_2024, wang_milora_2025, fan_make_2025} to keep only the top rank-$r$ directions—thereby restricting update rank and adaptation capacity.
Shared random-basis PEFT methods~\cite{kopiczko_vera_2023, koohpayegani_nola_2023} further cut parameters by reusing fixed random subspaces across layers, but all updates are confined to theses subspaces, 
limiting task-specific expressiveness and ultimately hampering final performance~\cite{li_vb-lora_2024, albert_randlora_2025}.
% While this approach nominally offers a high-rank update \cite{li_vb-lora_2024, albert_randlora_2025}, all parameter changes remain within that random subspace, restricting  exploration of task-specific directions and often reducing final accuracy.
% end

% To address these challenges, various PEFT techniques have been developed to reduce resource demands while maintaining task performance. Low-rank decomposition techniques improve efficiency by constraining updates to low-dimensional subspaces or fixed random bases. However, these approaches often involve trade-offs between parameter efficiency and adaptability, leading to slower convergence and accuracy gaps compared to full fine-tuning.

% Motivated by these observations, we propose FRoD (Full-Rank Efficient Fine-Tuning with Rotational Degrees). FRoD begins by jointly decomposing pretrained weight matrices across all layers to extract a single global orthonormal basis, along with layer-specific projection matrices and scaling factors that accurately reconstruct the original weights. In the fine-tuning phase, FRoD enriches these scaling factors by adding a learnable sparse perturbation, implicitly enabling controlled rotations within the shared latent space. This design preserves the core strength structure of the pretrained model while expanding the update space to improve expressiveness and convergence speed.

%% 这张图要体现第一性原理。 可能把这张图放到第一张图，作为抽象特征更好（只是一个想法）。
\begin{figure*}[t]
\centering
\includegraphics[width=\linewidth]{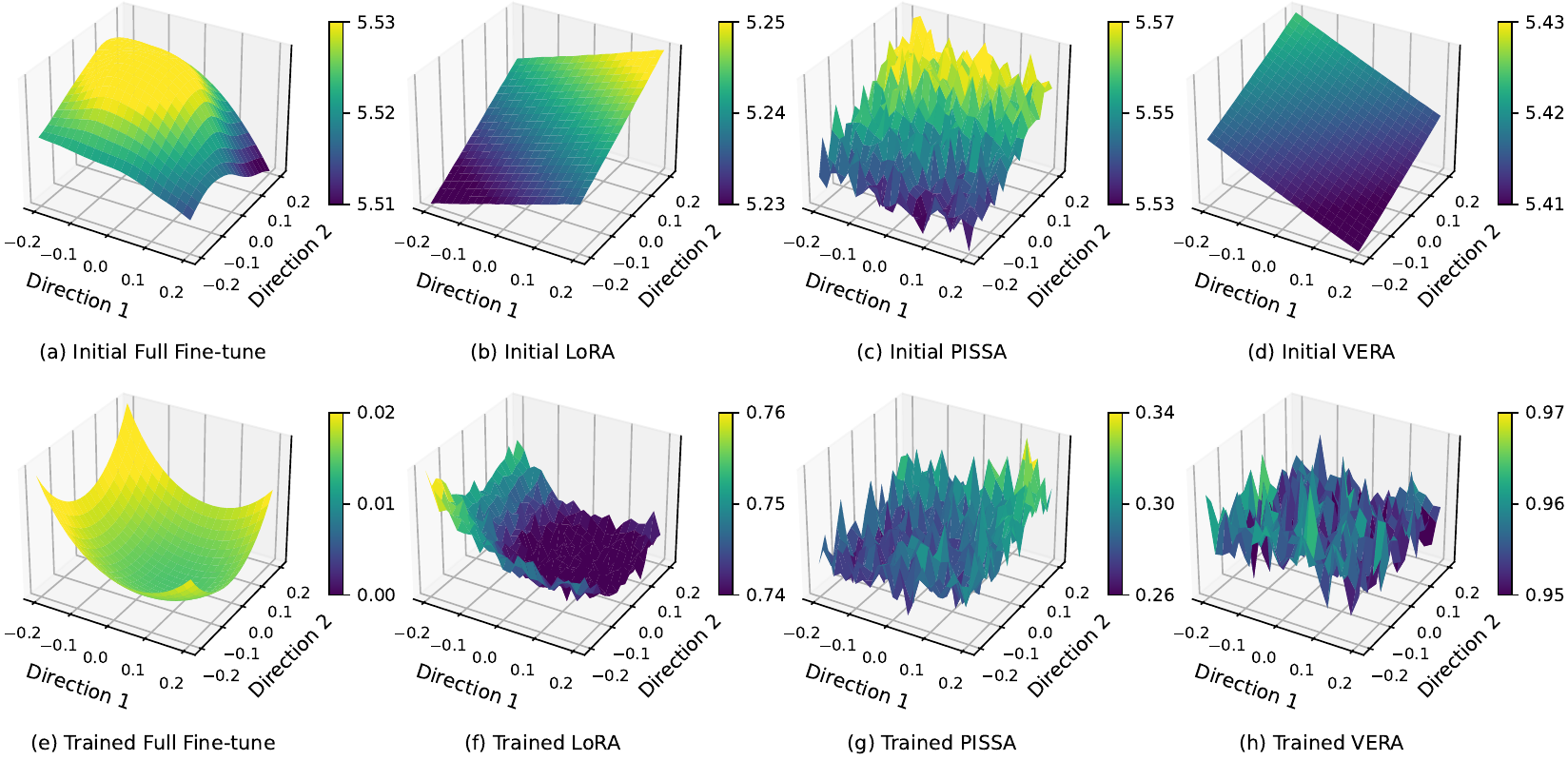}
\caption{
% Visualization of Loss Landscapes for Different Optimization Methods in Model Parameter Space. The figure illustrates the three-dimensional loss landscapes of four representative deep learning optimization methods in a two-dimensional parameter subspace. 
Comparison of loss landscapes for four representative fine-tuning methods in a principal parameter subspace.
% In each subplot, the horizontal axes represent two principal parameter directions (Direction $\alpha$ and Direction $\beta$), while the vertical axis denotes the loss value. 
Each subplot shows the loss surface as a function of two principal directions ($\alpha$ and $\beta$) in parameter space, with the vertical axis denoting the loss value. 
% The top row (a-d) depicts the loss surface at model initialization, whereas the bottom row (e-h) presents the loss landscape after training convergence. 
The top row (a–d) presents landscapes at model initialization; the bottom row (e–h) shows them after convergence. 
These visualizations provide a comparative geometric view of how various fine-tuning strategies shape the optimization landscape during training.
}
\label{fig:motivation}
\end{figure*}

%% 全文统一是me t h d还是method。basis and rotation degree of freedom应该具像化一些。
% To address this issue, we propose FRoD, a novel fine-tuning metho d that enhances both expressiveness and convergence by expanding the update space to full rank while keeping parameter overhead low. FRoD begins with a hierarchical joint decomposition that produces one shared basis and per‑layer diagonal strength matrices \(\Sigma_i\).  Adjusting only \(\Sigma_i\) yields on‑axis fine‑tuning, staying aligned with the pretrained axes at almost no parameter cost.  To move beyond those axes, we append a sparse, learnable matrix \(S_i\); its off‑diagonal entries introduce off‑axis interactions that act as localized rotations—exactly the sparse spaces and latent rotational degrees of freedom illustrated in Figure~\ref{fig:intro}(c).  This design allows flexible adaptation directions while preserving the strength structure of foundation weights, enabling fast and stable optimization across diverse tasks. This design allows flexible adaptation directions while preserving the strength structure of foundation weights, enabling fast and stable optimization across diverse tasks. Our contributions are:

We propose \textbf{\textit{FRoD}}, a full-rank, efficient fine-tuning method that couples strong expressiveness with fast convergence. FRoD starts by performing a hierarchical joint decomposition to extract a global shared basis and per-layer diagonal strength matrices $\Sigma_i$; updating only $\Sigma_i$ delivers near-zero-cost on-axis fine-tuning. 
To step beyond those axes, we append a sparse, learnable matrix $S_i$ to each layer. Its off-diagonal entries introduce off-axis interactions, injecting rotational degrees of freedom into the latent space (\figurename~\ref{fig:intro}c). 
This design maintains the foundational weight structure while offering flexible adaptation directions, supporting fast and stable optimization across diverse tasks.
% RX TODO: how about this desc
% This design preserves the compact structure while greatly expanding adaptation capacity, supporting fast and stable optimization across diverse tasks.

% \begin{itemize}
%   \item \textbf{Hierarchical Joint Decomposition.} We derive a unified orthonormal basis shared across layers via a combined QR and eigen-decomposition, ensuring high-fidelity reconstruction of pretrained weights.
%   \item \textbf{Rotational Degrees of Freedom.} We introduce sparse perturbations to the scaling factors, implicitly enabling local rotations in the latent space for flexible, full-rank updates with minimal overhead.
%   \item \textbf{Extensive Empirical Validation.} We verify FRoD's rapid convergence and its full-rank performance upper bound across a wide range of benchmark datasets.
% \end{itemize}

\begin{itemize}
    \item We propose a hierarchical joint decomposition method that extracts a shared latent basis across all model layer, 
    capturing model-wide structure from the outset.
    \item We introduce sparse, learnable updates to enable flexible off-axis adaptation, expanding the update space and improving both convergence and expressiveness.
    % \item Empirical results show that FRoD matches full fine-tuning performance within 0.09\% while using only 1.72\% of parameters on 20 diverse benchmarks.
    \item Extensive experiments on 20 vision, reasoning, and language benchmarks show that FRoD matches the accuracy of full fine-tuning, using just 1.72\% of the parameters.
\end{itemize}

\section{Background and Motivation}

% 我在思考是不是可以不用写Preliminary了
\subsection{Background: LoRA and Its Variants}

% \textbf{Low-Rank Adaptation (LoRA)} assumes the model update $\Delta W$ during adaptation is low-rank. Given pretrained weights $W_0 \in \mathbb{R}^{m \times n}$, LoRA injects trainable matrices $B \in \mathbb{R}^{m \times r}$, $A \in \mathbb{R}^{r \times n}$:
% \begin{equation}
% \label{eq: LoRA}
% W = W_0 + BA,
% \end{equation}
% with $r \ll \min(m, n)$, and $W_0$ frozen during training. LoRA suffers from slow initial convergence due to suboptimal initialization.

% \noindent \textbf{VeRA} fixes random matrices $A, B$ shared across layers and introduces trainable scaling vectors $b \in \mathbb{R}^m, d\in \mathbb{R}^r$:
% \begin{equation}
% W = W_0 + \Lambda_b B \Lambda_d A,
% \end{equation}
% where $\Lambda_b, \Lambda_d$ are diagonal matrices from vectors $b, d$. VeRA reduces parameters to $O(r)$ per layer and avoids redundant optimizer states.

% \noindent \textbf{PiSSA} initializes $A, B$ via truncated singular value decomposition (SVD) of $W_0$:
% \begin{equation}
% W_0 = U S V^\top,
% \end{equation}
% extracting principal components $W_{\text{pri}}$ and residuals $W_{\text{res}}$. Then:
% \begin{equation}
% B = U_{[:, :r]} S_{[:r, :r]}^{1/2}, \quad A = S_{[:r, :r]}^{1/2} V_{[:, :r]}^\top,
% \end{equation}
% allowing faster convergence by focusing on principal directions.

% 观点句开头:先阐述，再细化
PEFT methods, such as LoRA and its variants, adopt a unified adaptation scheme expressed as $W = W_0 + \Delta W$, where $W_0$ denotes frozen pretrained weights and $\Delta W$ defines the learnable adaptation.
Typically:
\begin{itemize}
    \item \textit{LoRA} replaces $\Delta W$ with the product of two trainable low-rank matrices, $BA$, reducing parameter costs.
    \item \textit{VeRA} extends LoRA by introducing learnable scaling vectors with shared random matrices, $\Delta W = \Lambda_b B \Lambda_d A$, further decreasing redundancy in optimizer states.
    \item \textit{PiSSA} imposes structural priors via truncated SVD: $W_0 = USV^\top$, initializing $B = U_{[:, :r]} S_{[:r, :r]}^{1/2}$ and $A  = S_{[:r, :r]}^{1/2} V_{[:, :r]}^\top$ using dominant singular directions to facilitate convergence along principal axes.
\end{itemize}
% LoRA-based methods are typically built upon a unified formulation of $W = W_0 + \Delta W$, where $W_0$ denotes the frozen pretrained weights, and different approaches design specific forms of $\Delta W$ to enable parameter-efficient adaptation. 

% Among them, the three most representative methods are LoRA, VeRA, and PiSSA. LoRA directly defines $\Delta W = BA$, using two trainable low-rank matrices to inject updates. VeRA constructs $\Delta W = \Lambda_b B \Lambda_d A$ by introducing learnable scaling vectors $b$ and $d$ on top of shared random matrices $A$ and $B$, aiming to reduce parameter overhead and eliminate redundant optimizer states. PiSSA incorporates structural priors by performing a truncated singular value decomposition $W_0 = USV^\top$, and initializes $B = U_{[:, :r]} S_{[:r, :r]}^{1/2}$ and $A = S_{[:r, :r]}^{1/2} V_{[:, :r]}^\top$ to form $\Delta W = BA$, focusing updates along principal singular directions to enhance convergence efficiency.

\begin{figure*}[!t]
\centering
\includegraphics[width=\linewidth]{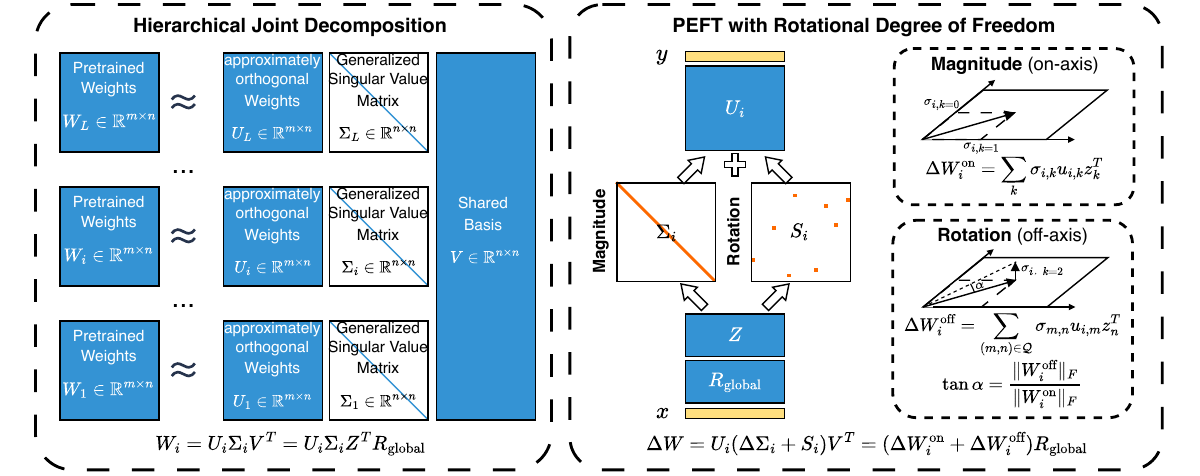}
\caption{Overview of FRoD. FRoD is a two-stage approach: hierarchical joint decomposition initializes the model by extracting a global shared basis, and sparse perturbations introduce rotational degrees of freedom during fine-tuning.}
\label{fig:method}
 
\end{figure*}

\figurename~\ref{fig:motivation} visualizes the loss landscapes under different PEFT schemes, highlighting their geometric properties before and after fine-tuning. 
This illustration motivates key aspects of our method; 
owing to space limitations, detailed discussion of related work is relegated to the Appendix A \& B.
% RX TODO: 是否要link 到 Appendix具体Section？
% We illustrate the curvature of the loss landscape before and after training under different PEFT schemes by plotting Fig.~\ref{fig:motivation}. This visualization helps motivate the core design choices of our method.

\subsection{Motivation: Convergence Efficiency}

% RX TODO 检查overleaf标横线单词，是否有几简单的拼写错误。

% 观点句开头:先阐述，再细化
The convergence efficiency of PEFT schemes is tightly linked to their loss landscape geometry at initialization. 
% RX TODO is this statement correct? or any other statements?
As shown in \figurename~\ref{fig:motivation}b, LoRA's low-norm random initialization yields a smooth, isotropic loss surface, supporting larger learning rates and stable early optimization. 
In contrast, PiSSA (\figurename~\ref{fig:motivation}c) aligns updates with dominant singular directions, sharpening the curvature and constraining the attainable step size, often resulting in slower convergence. 
Despite distinct initializations, both LoRA and PiSSA ultimately converge to similarly sharp optima (\figurename~\ref{fig:motivation}f--g), reflecting the inherent limitations of low-rank update spaces for capturing task complexity.
In contrast, VeRA (\figurename~\ref{fig:motivation}d) sustains a smooth loss surface akin to LoRA, enabling aggressive learning rates. 

These observations emphasize the critical role of structured initialization in shaping early convergence dynamics. 
Inspired by these findings, we propose a hierarchical joint decomposition framework that globally extracts shared subspaces. 
This approach supports invariant, structure-aware initialization across layers, providing a principled trade-off between convergence speed and parameter efficiency.

\subsection{Motivation: Parameter Degree of Freedom}

Full fine-tuning enables unrestricted optimization in the entire parameter space, exhibiting superior convergence speed and expressive capacity. In contrast, PEFT methods, despite their lower training and deployment costs, are often confined to limited update subspaces. This constraint not only impairs the linear descent trajectory but also compresses the representational flexibility in later stages, hindering adaptation to complex tasks. 
% As illustrated in \figurename~\ref{fig:motivation}, LoRA and VeRA both adopt zero initialization, yet their curvature behaviors differ significantly. LoRA benefits from a more relaxed update space, which maintains stable optimization (\figurename~\ref{fig:motivation}f). However, due to its highly constrained update subspace, the loss surface of VeRA becomes increasingly distorted during training, reflecting the negative impact of insufficient degrees of freedom on convergence stability and final performance.
% 对比陈述 更简洁抽象一些 RX  TODO 准确性？
As depicted in \figurename~\ref{fig:motivation}, 
LoRA and VeRA share similar (zero) initialization but differ in update space geometry: 
LoRA preserves relative smoothness and flexibility (\figurename~\ref{fig:motivation}f), whereas VeRA’s highly constrained subspace results in a distorted loss landscape and reduced convergence stability (\figurename~\ref{fig:motivation}d).

% RX TODO  degree of freedom说法严谨性？是否常用说法？or degree of rotation?
These phenomena suggest that it is not initialization alone but, more fundamentally, the dimensionality of the effective update space, often termed the Parameter Degrees of Freedom (PDoF), that determines adaptation capability.
Enhancing PDoF requires expanding the space of feasible updates, preferably in a principled, structured way. 
We operationalize this via rotations in both the global shared basis and its orthogonal complements. 
While direct rotation via dense orthogonal matrices is impractical in high dimensions, we approximate this mechanism through additive sparse matrices, which efficiently inject rotational flexibility with minimal computational overhead.

\section{Method}

FRoD enhances parameter-efficient fine-tuning through a two-stage framework designed to maximize both convergence speed and expressiveness, as illustrated in \figurename~\ref{fig:method}.
% FRoD enables efficient fine-tuning and enhanced expressiveness through a two-stage framework. 
First, we perform a hierarchical joint decomposition that uncovers a globally shared subspace across layers, while preserving complementary layer-specific orthogonal components, yielding a principled and unified parameter initialization.
% In the first stage, we perform hierarchical joint decomposition to extract a globally shared basis across layers while preserving layer-specific orthogonal structures, resulting in a unified initialization scheme.
Building upon this foundation, FRoD introduces a novel rotational PEFT mechanism that decomposes parameter updates into magnitude changes along the original directions and sparse off-axis rotations. This design injects rotational degrees of freedom into the update space, effectively expanding model expressiveness while maintaining parameter efficiency.

\begin{table*}[t]
\centering

\begin{tabular}{lccc}
\toprule
 & FRoD & LoRA & VeRA \\
\midrule
Weights & $L(mn + s\,n^2 + n) + n^2 $ & $L(mn + mr + nr)$ &$ L(mn + r + n) + mr + nr $ \\
Optim States & $2 L(s\,n^2 + n)$ & $2 L(m r + n r)$ & $2 L(r + n)$ \\
\bottomrule
\end{tabular}
\caption{Trainable parameter and optimizer states comparison. Assume $W \in \mathbb{R} ^ {m \times n}$, rank $r$.}
\label{tab:mem-comparison}
\end{table*}

% TODO: Fig3 的介绍

\subsection{Hierarchical Joint Decomposition}
We propose a hierarchical joint decomposition-based initialization scheme by performing joint decomposition \cite{kempf_higher-order_2023} across layers and categories. 

Given a set of pre-trained matrices $\{ W_i^{(c)} \}$, where $i$ indexes the layer and $c \in \mathbb{C}$ denotes the category, our objective is to extract a shared latent basis for subsequent fine-tuning.
% The goal is to extract a shared basis from the pre-trained weight matrices \( \{ W_{i}^{(c)} \} \), where \( i \) indexes the layer and \( c \in \mathbb{C} \) denotes the category.

For each category $c$, we concatenate the weight matrices $W_i^{(c)} \in \mathbb{R}^{m \times n}$ from all $L$ layers vertically:
% For each category \( c \), we collect matrices from all layers: $W_{i}^{(c)} \in \mathbb{R}^{m \times n}$, and vertically stack them as
\begin{equation}
% W^{(c)} = \begin{bmatrix} W_{1}^{(c)} \\ W_{2}^{(c)} \\ \vdots \\ W_{L}^{(c)} \end{bmatrix} \in \mathbb{R}^{M_c \times n},
W^{(c)} = \begin{bmatrix} W_{1}^{(c)}  W_{2}^{(c)} \cdots  W_{L}^{(c)} \end{bmatrix}^T \in \mathbb{R}^{M_c \times n},
\end{equation}
where $M_c = L \cdot m$. % RX TODO: need to check correctness.

A thin QR decomposition yields:
\begin{equation}
W^{(c)} = Q^{(c)} R_\text{global},
\end{equation}
with $Q^{(c)} \in \mathbb{R}^{M_c \times n}$ orthonormal and $R_\text{global} \in \mathbb{R}^{n \times n}$ upper triangular.
Each $Q^{(c)}$ is partitioned into layer-wise blocks $Q_i$ according to the row dimensions of $W_{i}^{(c)}$.
% Each \( Q^{(c)} \) is then split into layer-wise segments \( Q_i \) according to the row dimensions of \( W_{i}^{(c)} \).

Next, we regularize and aggregate the category-wise projections by computing:
\begin{equation}
T_\pi = \frac{1}{|\mathbb{C}|} \sum_{c \in \mathbb{C}} \left( (Q^{(c)})^\top Q^{(c)} + \pi I \right)^{-1},
\end{equation}
where \( \pi > 0 \) is a small constant to stabilize the inversion.
We diagonalize $T_\pi$ as:
\begin{equation}
T_\pi = Z \Lambda Z^\top,
\end{equation}
where $Z \in \mathbb{R}^{n \times n}$ is orthogonal. 
The resulting shared latent basis is given by:
\begin{equation}
V^\top = Z^\top R_\text{global}.
\end{equation}

% There exists an orthogonal $Z \in  \mathbb{R}^ {n\times n} $such that :
% \begin{equation}
% T_\pi = Z \Lambda Z^\top,
% \end{equation}
% and obtain the shared orthonormal latent basis:
% \begin{equation}
% V = (R_\text{global})^\top Z \in \mathbb{R}^{n \times n}.
% \end{equation}

For each layer-category pair, we compute the local full-rank representation:
\begin{equation}
B_i = Q_i Z, \quad \Sigma_i = \mathrm{diag}(\|B_i\|_2), \quad U_i = B_i \Sigma_i^{-1},
\end{equation}
% \begin{equation}
% B_i = Q_i Z,\quad 
% \Sigma_i = \mathrm{diag}\big(\|B_i^{(:,j)}\|_2\big), \quad 
% U_i = B_i \Sigma_i^{-1},
% \end{equation}
which yields the approximation:
\begin{equation}
W_{i}^{(c)} \approx U_i \Sigma_i V^\top.
\label{eq:w_decom}
\end{equation}

% This decomposition provides a global basis \( V \), shared across all layers, with local orthonormal projections \( U_i \) and scaling factors \( \Sigma_i \). 

% Such structure enables effective and consistent initialization of downstream adaptation modules in PEFT.

% Importantly, the hierarchical joint decomposition maximally preserves the strength components along every singular direction of the original weight matrices. Empirically, the per-layer reconstruction error, measured by the Frobenius norm \( \|W_i - U_i \Sigma_i V^\top\|_F \), remains below \(1\times10^{-5}\). In subsequent experiments, we leverage this high-fidelity reconstruction to initialize our fine-tuning modules, demonstrating its impact on convergence and performance.

% update 8.2
This hierarchical joint decomposition extracts a global orthonormal basis $V$ across categories and layers, while the local orthogonal components $U_i$ and scaling factors $\Sigma_i$ retain layer-specific information. 
The reconstructed factors provide high-fidelity initialization for our fine-tuning modules.
% During hierarchical joint decomposition, we preserve a global basis \( V \), retain the magnitudes of the generalized singular values \( \Sigma_i \), and maintain the approximate orthogonality of $U_i$. We subsequently employ this high-fidelity reconstruction to initialize the fine-tuning modules.
% end

% , and our experiments demonstrate that the resulting initialization significantly accelerates convergence and improves final performance.

% \subsection{PEFT with Rotational Degree of Freedom}
% \label{sec:rdof}
% Building upon the joint decomposition described in the previous section, we propose a fine-tuning method that augments the factorized form with a learnable sparse perturbation. Specifically, instead of directly using the decomposition as shown in Eq. 7. We introduce a new trainable matrix \( S \in \mathbb{R}^{n \times n} \) and redefine the adaptation matrix as
% \begin{equation}
% W_i' = U_i (\Sigma_i + S_i) V^\top,
% \end{equation}
% where \( \Sigma_i \) is a fixed diagonal matrix derived from the initial decomposition, and \( S \) is a sparse, low-magnitude, learnable perturbation.

\subsection{PEFT with Rotational Degree of Freedom}
\label{sec:rdof}

Building upon the joint decomposition described in the previous section, we propose a PEFT method that augments the factorized form with a learnable sparse perturbation. 
% RX TODO check equation 7?
Instead of directly using the decomposition in Eq.~\ref{eq:w_decom}, We introduce a new trainable matrix \( S_i \in \mathbb{R}^{n \times n} \) and redefine the adaptation matrix as
\begin{equation}
W_i' = U_i (\Sigma_i + S_i) V^\top,
\end{equation}
where $S_i$ adopts random off-diagonal sparsity matrix, and $s = \mathrm{nnz}(S_i)/n^2$ denotes its density, enabling rotation in the latent space with minimal parameter overhead.

\noindent\textbf{Rotational Degree of Freedom(RDoF).} 
The sparse perturbation \( S_i \) introduces a RDoF in the latent space. For an input vector \( x \in \mathbb{R}^n \), the original transformation is:
\begin{equation}
y = W_i x = U_i \Sigma_i V_i^\top x.
\end{equation}

Let \( Z \in \mathbb{R}^{n \times n} \) be a transformation matrix, and define the latent coordinates as \( c_v = R_{\text{global}} x \), where \( R_{\text{global}} \in \mathbb{R}^{n \times n} \) is an orthogonal basis. 
The transformation becomes:
\begin{equation}
    y = U_i \Sigma_i Z^\top c_v,
\end{equation}
representing an axis-aligned scaling in the latent space.

With perturbation \( S_i \), the adapted transformation becomes
\begin{equation}
    y' = W_i' x = U_i (\Sigma_i + S_i) V_i^\top x = U_i (\Sigma_i + S_i) Z^\top c_v.
\end{equation}

Thus, the update matrix is 
\begin{equation}
    \Delta W_i = W_i' - W_i = U_i (\Delta \Sigma_i +S_i)Z^\top R_\text{global},
\end{equation}
which can be decomposed into orthogonal components:
\begin{itemize}
    \item \textbf{Magnitude (on-axis)}: \( \Delta W_i^{\text{on}} = U_i \Delta \Sigma_i Z^\top \), adjusting the generalized singular values (intensity).
    \item \textbf{Rotation (off-axis)}: \( \Delta W_i^{\text{off}} = U_i S_i Z^\top \), introducing cross-dimensional coupling via the off-diagonal \( S_i \).
\end{itemize}

These components satisfy an orthogonality condition and thereby enable a geometric decomposition of the update.
By explicitly modeling this rotational degree of freedom with a sparse, learnable matrix, our method significantly increases the expressiveness of adaptation while retaining strong parameter efficiency. % RX: added a short conclusion

\begin{table*}[t]
\centering
\begin{tabular}{lccccccccc}
\toprule
Method & \# Params (\%) & Cars & DTD & EuroSAT & GTSRB & RESISC45 & SUN397 & SVHN & Average \\
\midrule
Full FT & 100 & 60.33 & 73.88 & 98.96 & 98.30 & 93.65 & 53.84 & 96.78 & 82.25 \\
LoRA $\dagger$ & 1.49 & 41.02 & 70.15 & 98.66 & 96.51 & 90.38 & 47.51 & 95.39 & 77.09 \\
\midrule
PiSSA $\dagger$ & 1.49 & 40.41 & 69.62 & 98.48 & 95.84 & 90.58 & 47.21 & 95.84 & 76.85 \\
MiLoRA $\dagger$ & 1.49 & 39.77 & 70.48 & 98.19 & 97.52 & 89.92 & 45.38 & 95.49 & 76.68 \\
GOAT $\dagger$ & 2.24 & 53.50 & \textbf{75.32} & 98.82 & \textbf{98.17} & \textbf{93.46} & 54.53 & \textbf{96.62} & 81.49 \\
\midrule
% LoRA (rank16) $\dagger$ & 2.99 & 46.51 & 72.07 & 98.74 & 98.04 & 92.08 & 51.63 & 96.00 & 79.30 \\
% LoRA (rank32) $\dagger$ & 5.98 & 50.13 & 72.87 & 98.88 & 98.13 & 92.87 & 53.65 & 96.55 & 80.44 \\

VeRA & 0.29 & 60.37 & 73.03 & 98.44 & 97.30 & 89.71 & 50.25 & 95.93 & 80.71\\
RandLoRA & 1.49 & 42.36 & 69.68 & 98.37 & 96.88 & 89.03 & 47.47 & 95.60 & 77.06\\
DoRA $\dagger$ & 1.49 & 40.75 & 71.91 & \textbf{98.89} & 97.71 & 90.19 & 47.54 & 95.46 & 77.49 \\
\midrule
FRoD(s=0.01) & 1.29 & 61.30 & 71.22 & 98.52 & 97.66 & 91.43 & \textbf{59.62} & 96.01 & 82.25 \\
% \rowcolor{red!15}
FRoD(s=0.02) & 2.49 & \textbf{62.13} & 73.24 & \textbf{98.89} & 97.84 & 91.95 & 59.16 & 96.42 & \textbf{82.80} \\
\bottomrule
\end{tabular}
\caption{We evaluate CLIP ViT-B/32 across StanfordCars, DTD, EuroSAT, GTSRB, RESISC45, SUN397, and SVHN datasets. The symbol $\dagger$ indicates that the results are taken from \cite{fan_make_2025}.}
% \caption{We evaluate CLIP ViT-B/32 with full fine-tuning, VeRA with total rank 128, RandLoRA with total rank 64 across StanfordCars, DTD, EuroSAT, GTSRB, RESISC45, SUN397, and SVHN datasets. The symbol $\dagger$ indicates that the results are taken from \cite{fan_make_2025}.}
\label{table:IC}
\end{table*}

\noindent\textbf{Preservation of Singular-Value Spectrum.}
% To firmly ground rotational PEFT, 
To provide a theoretical foundation for our rotational PEFT framework,
we first establish a spectral-stability theorem for sparse perturbations, 
followed by two corollaries: 
(i) an orthogonal decomposition of on-axis and off-axis updates, 
and (ii) an angular representation of the total update. 
Collectively, these results show that sparse off-axis rotations keep the dominant singular-value directions essentially intact while expanding the update space in a controlled way. 
% This provides the mathematical foundation for the convergence analysis that follows. % RX 和 firmly ground或者theroretical foundation重复

\begin{theorem}[Spectral Stability Under Sparse Perturbations]
\label{thm:spectral_stability}
Let $\Sigma_i\in\mathbb{R}^{n\times n}$ be a diagonal matrix of singular values and let  
$S_i\in\mathbb{R}^{n\times n}$ be a perturbation such that  
\begin{equation}
\|S_i\|_0\ll n^2 
\quad\text{and}\quad
\max_{p,q}|s_{pq}|\le\varepsilon.
\end{equation}
Then, for every $k$,
\begin{equation}
\bigl|\sigma_k(\Sigma_i+S_i)-\sigma_k(\Sigma_i)\bigr|
\;\le\;
\|S_i\|_2 ,
\end{equation}
so the dominant singular directions encoded by $\Sigma_i$ are preserved up to a deviation bounded by $\|S_i\|_2$.
\end{theorem}

\begin{corollary}[Orthogonal Decomposition of Updates]
\label{cor:orthogonality}
Define the on-axis and off-axis updates as
\(
\Delta W_i^{\text{on}},\;\Delta W_i^{\text{off}}\in\mathbb{R}^{m\times n}.
\)
Then
\begin{equation}
\langle\Delta W_i^{\text{on}},\Delta W_i^{\text{off}}\rangle_F \;=\; 0 ,
\end{equation}
i.e.\ the two update components are orthogonal in the Frobenius inner product.
\end{corollary}

\begin{corollary}[Angular Representation of the Total Update]
\label{cor:angular_representation}
Let
\(
\tan\alpha=\|\Delta W_i^{\text{off}}\|_F/\|\Delta W_i^{\text{on}}\|_F
\)
and set
\(
\hat U_{\text{on}}=\Delta W_i^{\text{on}}/\|\Delta W_i^{\text{on}}\|_F,\;
\hat U_{\text{off}}=\Delta W_i^{\text{off}}/\|\Delta W_i^{\text{off}}\|_F.
\)
Then
\begin{equation}
\Delta W_i
=\|\Delta W_i\|_F\bigl(\cos\alpha\,\hat U_{\text{on}}
+\sin\alpha\,\hat U_{\text{off}}\bigr).
\end{equation}

For small $\alpha$ (when $\|\Delta W_i^{\text{on}}\|_F\gg\|\Delta W_i^{\text{off}}\|_F$),
\begin{equation}
\Delta W_i \approx \|\Delta W_i^{\text{on}}\|_F \left( \hat{U}_{\text{on}} + \alpha \cdot \hat{U}_{\text{off}} \right) R_\text{global}.
\end{equation}

Thus, the total update can be interpreted as a rotation from the axis-aligned direction by a small angle $\alpha$, reflecting a controllable expansion as Shared Random-Basis PEFT. 
\end{corollary}

Detailed proofs are provided in the Appendix C.

\subsection{Memory Usage of Projection Matrices}

In FRoD, the factorized form Eq. 8 uses frozen matrices $U_i\in\mathbb{R}^{m\times n}$ and
shared $V\in\mathbb{R}^{n\times n}$, while only $\Sigma_i\in\mathbb{R}^n$ and sparse $S_i\in\mathbb{R}^{n\times n}$ are trainable. Across $L$ layers, the total number of trainable parameters is $L(mn+3sn^2 +n) + n^2$.
For comparison, LoRA with rank $r$ has $ L(mn + 3mr + 3nr)$ and VeRA shares fixed random bases $A\in\mathbb{R}^{m\times r}$, $B\in\mathbb{R}^{r\times n}$ across layers
and trains only scaling vectors $d\in\mathbb{R}^r, b \in \mathbb{R}^n$, yielding $ L(mn + 3r + 3n) + mr + nr $. A comparison between FRoD, LoRA and VeRA is shown in \tablename~\ref{tab:mem-comparison}.

\begin{table*}[t]
\centering

\begin{tabular}{lcccccccccc}
\toprule
Method & \# Params(\%) & BoolQ & PIQA & SIQA & HellaSwag & WinoGrande & OBQA & Average \\
\midrule
ChatGPT $\dagger$ & / & 73.10 & 85.40 & 68.50 & 78.50 & 66.10 & 74.80 & 77.01 \\
Full FT & 100 & 88.50 & 83.79 & 81.16 & 92.50 & 82.95 & 78.60 & 84.58 \\
LoRA$\dagger$ & 0.84 & 69.80 & 79.90 & 70.50 & 83.60 & 82.60  & 81.00 & 77.61 \\

\midrule
PiSSA$\dagger$ & 0.84 & 67.60 & 78.00 & 71.80 & 78.00 & 75.80  & 75.60 & 73.78 \\
MiLoRA$\dagger$ & 0.84 & 67.60 & 83.30 & 73.40 & 88.20 & 83.00  & 80.80 & 79.24 \\
% KaSA$\dagger$ & 0.84 & 73.60 & 84.40 & 80.70 & 91.50 & 84.50 & 81.20 & 81.53 \\
GOAT $\dagger$ & 0.96 & 73.60 & \textbf{83.95} & 80.55 & 80.50 & \textbf{85.00} & \textbf{87.00} & 82.73 \\
\midrule
VeRA & 0.02 & 77.68 & 69.58 & 84.73 & 87.99 & 50.00 & 30.20 & 66.70 \\
RandLoRA & 0.84 & 87.65 & 82.70 & 88.05 & 92.91 & 75.77 & 34.20 & 76.88 \\
DoRA$\dagger$ & 0.84 & 71.80 & 83.10 & 79.90 & 89.10 & 83.00  & 81.20 & 80.45 \\
\midrule
% \rowcolor{red!15}
FRoD(s=0.02) & 0.18 & \textbf{87.74} & 83.62 & \textbf{81.17} & \textbf{93.28} & 80.98 & 74.40 & \textbf{83.53} \\
\bottomrule
\end{tabular}
\caption{Performance comparison of LLaMA2 7B on 6 commonsense reasoning datasets. The symbol $\dagger$ indicates that the results are taken from \cite{fan_make_2025}.}
\label{table:CR}

\end{table*}

\begin{table*}[h]
\centering
\begin{tabular}{lccccccccc}
\toprule
Method & \# Params (\%) & CoLA & SST-2 & MRPC & QQP & MNLI & QNLI & RTE & Average \\
\midrule
Full FT & 100 & 84.27 & 96.44 & 90.44 & 91.58 & 90.84 & 94.58 & 84.84 & 90.42 \\
LoRA $\dagger$ & 4.00 & 83.41 & 95.64 & 83.33 & 90.06 & 89.00 & 93.28 & 84.47 & 88.46 \\
\midrule
PiSSA $\dagger$ & 4.00 & 69.12 & 95.98 & 82.84 & 91.24 & 88.94 & 93.59 & 73.29 & 85.00 \\
MiLoRA $\dagger$ & 4.00 & 84.65 & 96.10 & 86.02 & 91.33 & 89.51 & 94.12 & 84.83 & 89.51 \\
GOAT $\dagger$ & 4.50 & \textbf{86.86} & 96.21 & 84.55 & 91.40 & 89.55 & 94.19 & \textbf{85.56} & 89.76 \\
\midrule
VeRA & 0.22 & 81.78 & 95.76 & 86.59 & 89.30 & 89.48 & 94.39 & 75.09 & 87.48\\
RandLoRA & 0.84 & 84.66 & 96.10 & 89.71 & 90.24 & \textbf{89.58} & 94.25 & 80.14 & 89.24\\
DoRA $\dagger$ & 4.00 & 85.33 & 95.99 & 84.07 & 91.24 & 89.52 & 93.54 & 84.48 & 89.17 \\
\midrule
% \rowcolor{red!15}
FRoD(s=0.02) & 2.49 & \textbf{86.86} & \textbf{96.33} & \textbf{90.44} & \textbf{91.63} & \textbf{89.58} & \textbf{94.40} & 84.48 & \textbf{90.53} \\
\bottomrule
\end{tabular}
\caption{Performance comparison of RoBERTa-large with different methods on 7 GLUE tasks. The symbol $\dagger$ indicates that the results are taken from \cite{fan_make_2025}.}
\label{table:NLU}

\end{table*}

\section{Experiments}

\subsection{Baselines}

We compare FRoD with a wide range of strong baselines, grouped into four categories, to comprehensively evaluate its effectiveness and robustness.
\textbf{Basic Mehtods}: \textit{Full FT} - fine-tunes all parameters; \textit{LoRA} \cite{hu_lora_2021}; \textit{ChatGPT} \cite{brown_language_2020}.
\textbf{SVD-Based LoRA Methods}:  \textit{PiSSA} \cite{meng_pissa_2024}; \textit{MiLoRA} \cite{wang_milora_2025}; \textit{GOAT} \cite{fan_make_2025}. 
\textbf{Shared Random-Basis PEFT Methods}: \textit{VeRA} \cite{kopiczko_vera_2023}; \textit{RandLoRA} \cite{albert_randlora_2025}. 
\textbf{Architecture-based Methods}: \textit{DoRA} \cite{liu_dora_2024}.

% \begin{itemize}
%     \item Full-Finetuning: Full FT fine-tunes all parameters, while Full FT MoE is Upcycled MoE with full-rank fine-tuning and 2 active experts out of 8 total experts. 
%     \item LoRA-based baselines: LoRA \cite{hu_lora_2021} ; PiSSA \cite{meng_pissa_2024}; MiLoRA \cite{wang_milora_2025}; KaSA \cite{wang_kasa_2025} ;GOAT \cite{fan_make_2025}.
%     \item Architecture-based baselines: DoRA \cite{liu_dora_2024}; VeRA \cite{kopiczko_vera_2023}; RandLoRA \cite{albert_randlora_2025}.
% \end{itemize}

\subsection{Datasets}

We evaluate FRoD across 20 tasks, spanning 3 domains.
\textbf{Image Classification (IC)}: We fine-tune and evaluate Clip ViT-B/32 \cite{radford_learning_2021} on 7 image classification datasets. 
\textbf{Commonsense Reasoning (CR)}: We fine-tune LLaMA2-7B \cite{touvron_llama_2023} on Commonsense170K and evaluate on 6 commonsense reasoning datasets \cite{hu_llm-adapters_2023}. 
\textbf{Natural Language Understanding (NLU)}: We fine-tune RoBERTa-large \cite{liu_roberta_2019} on 7 GLUE tasks \cite{wang_glue_2019}, following the setup of \cite{hu_lora_2021}.  

% \begin{itemize}
%     \item Image Classification (IC): We fine-tune and evaluate Clip ViT-B/32 \cite{radford_learning_2021} on 7 image classification datasets.  
%     \item Commonsense Reasoning (CR): We fine-tune LLaMA2-7B \cite{touvron_llama_2023} on Commonsense170K and evaluate on 6 commonsense reasoning datasets \cite{hu_llm-adapters_2023}. 
%     \item Natural Language Understanding (NLU): We RoBERTa-large \cite{liu_roberta_2019} on 7 GLUE tasks \cite{wang_glue_2019} following \cite{hu_lora_2021}.  
% \end{itemize}

\subsection{Main Results}

Tables 2, 3 and 4 present results on 3 domain benchmarks:
\begin{itemize}
    \item \textbf{Image classification (\tablename~\ref{table:IC}).} FRoD attains the best \emph{average} accuracy across seven datasets in just three epochs, edging out full fine‑tuning by 0.4 accuacy with only 2.5\% trainable parameters. GOAT outperforms FRoD on two individual datasets, a gap we attribute to its dataset‑specific rank tuning for a detailed comparison.
    \item \textbf{Commonsense reasoning (\tablename~\ref{table:CR}).} Within four epochs, FRoD delivers decisive gains over full FT on BoolQ, PIQA, SIQA and HellaSwag while maintaining a 500× parameter reduction. Lower scores on WinoGrande and OBQA stem from limited training examples and prompt-template designing.
    \item \textbf{Natural language understanding (\tablename~\ref{table:NLU}).} Running for only three epochs, FRoD matches or surpasses full FT on five GLUE tasks. The single shortfall on RTE is largely due to the tight epoch budget, extending training to ten epochs recovers parity.
\end{itemize}

% RX TODO: 结论里面GOAT强调多，没有强调FRoD？
In summary, FRoD consistently delivers leading performance across all benchmarks, excels in nearly every sub-task, and narrows the remaining gap to full fine-tuning and, in some cases, even eliminates it, thereby attesting to the heightened efficacy of our method. All metrics are averaged over three independent training runs on each dataset. Further prompt-template details and analyses are presented in Appendix D.

\subsection{Convergence Speed}

% \begin{table}[t]
% \centering
% \caption{Hyperparameters of the image classification task for FRoD.}
% \renewcommand{\arraystretch}{1.2}
% \begin{tabular}{lccccccc}
% \toprule
% \textbf{Hyperparameter} & IC \\
% \midrule
% Batch Size & 128 \\
% Optimizer & AdamW \\
% Warmup Steps & 0.1 \\
% LR Schedule & Cosine \\
% Learning Rate($\Sigma_i$) & 5e-3 \\
% Learning Rate($S$) & 5e-4 \\
% Epochs &   3 \\
% \bottomrule
% \end{tabular}
% \label{table: Hyperparameters}
% \end{table}

% Table~\ref{table: Hyperparameters} compares the hyperparameter settings of FRoD and the SOTA GOAT method on image classification tasks. Across most datasets, 
% During our main experiments, we observed that FRoD exhibits exceptionally strong convergence behavior across various benchmarks. Specifically, for image classification tasks, FRoD converges rapidly in just 3 epochs, while the state-of-the-art baseline GOAT typically requires over 10 epochs to reach comparable performance. Notably, on the SVHN dataset, FRoD achieves convergence in only 1 epoch, significantly faster than GOAT’s 4 epochs. This remarkable efficiency highlights the superior convergence properties of FRoD. Similar accelerated convergence trends are also evident in the commonsense reasoning and natural language understanding (NLU) experiments, where FRoD reaches peak accuracy within {3,10} epochs (10 epoches for RTE), compared to the 3–100 epochs required by other baseline methods. Detailed experimental results supporting these observations are provided in the Appendix.
FRoD exhibits markedly faster convergence across the main experiments. It reaches convergence within 4 epochs on most IC benchmarks and within a single epoch on SVHN, whereas the strongest LoRA-based baseline, GOAT, requires more than ten epochs. On CR and NLU tasks, FRoD achieves peak validation accuracy in 3 epochs, except for RTE, which stabilizes after ten, whereas competing methods typically require between three and one hundred epochs. The complete set of training hyper-parameters is provided in the Appendix D.

\begin{figure}[t]
\centering
\includegraphics[width=\linewidth]{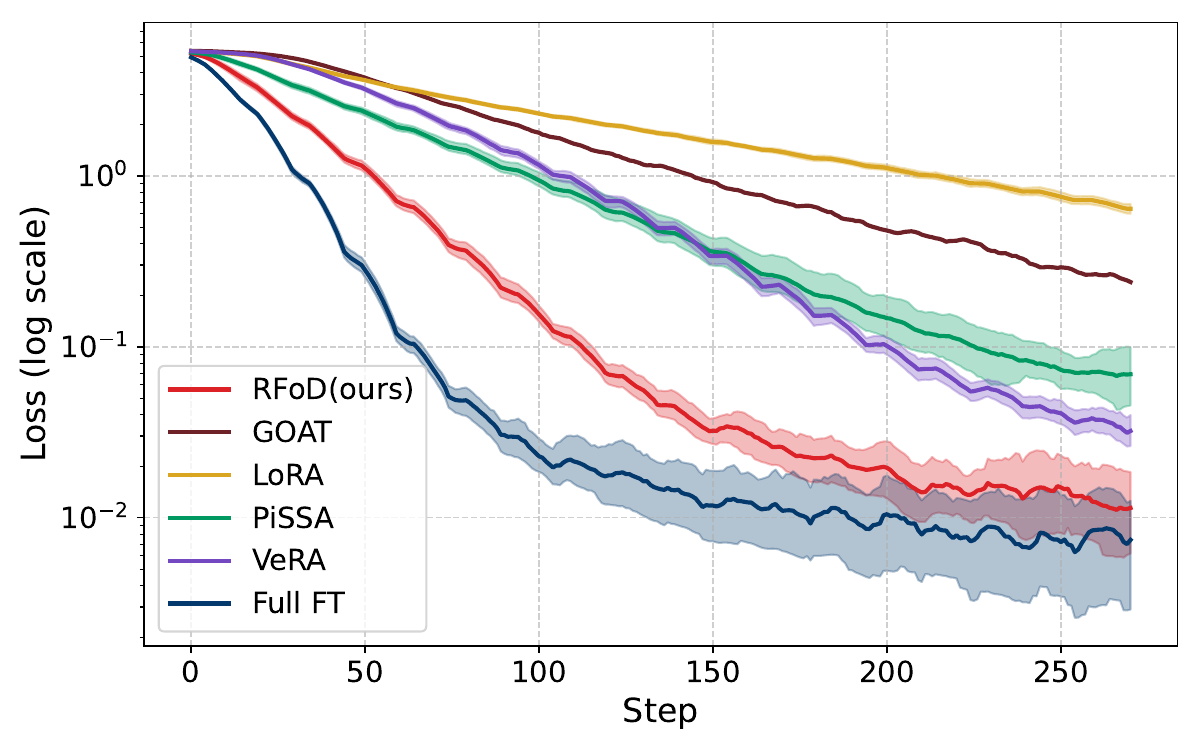}
\caption{Training loss curves of Different LoRA methods and Full Fine-tuning on Cars. The shaded areas in the figure represent the error bounds of different methods. We randomly sample five seed values. }
\label{fig:converage}
\end{figure}

\figurename~\ref{fig:converage} illustrates the epoch-wise validation accuracy on Stanford Cars for FRoD under five different random seeds. The solid curve denotes the mean accuracy across seeds, while the shaded region represents the min–max range of those runs. Across these runs, FRoD’s final accuracy remains within $0.1\%$ of full fine-tuning, demonstrating that our convergence speed and performance are stable and insensitive to initialization variance.

% \begin{figure*}[t]
% \centering
% \includegraphics[width=\linewidth]{figs/Value Reconstruction Error.pdf}
% \caption{An illustrative motivation figure.}
% \label{fig:param_effectiveness}
% \end{figure*}

% \subsection{Hierarchical Joint Decomposition Result Analysis}
\subsection{Analysis of Loss Landscape Geometry}
% \begin{figure}[t]
% \centering
% \includegraphics[width=\linewidth]{figs/random_init_bar.pdf}
% \caption{An illustrative motivation figure.}
% \label{fig:random_init_bar}
% \end{figure}
% Figure~\ref{fig:param_effectiveness} visualizes the per-layer singular value spectra for PiSSA initialization versus our initialization. PiSSA’s layer-wise truncated SVD often omits important large singular directions, leading to uneven signal capture and slower convergence. In contrast, FRoD jointly decomposes across all layers, ensuring that each layer retains its dominant singular modes. As a result, the minimum leading singular value across layers is significantly higher under Joint Decomposition, demonstrating its superior ability to preserve parameter effectiveness in all subspaces.

% FRoD’s hierarchical joint decomposition significantly reduces the initialization error compared to the original parameter matrices, achieving reconstruction accuracy within an order of 1e-5. In contrast, PiSSA’s layer-wise truncated SVD often neglects critical singular directions, limiting initial parameter effectiveness and convergence speed. This highlights FRoD’s superior ability to retain dominant singular structures uniformly across layers.

\begin{figure}[t]
\centering
\includegraphics[width=\linewidth]{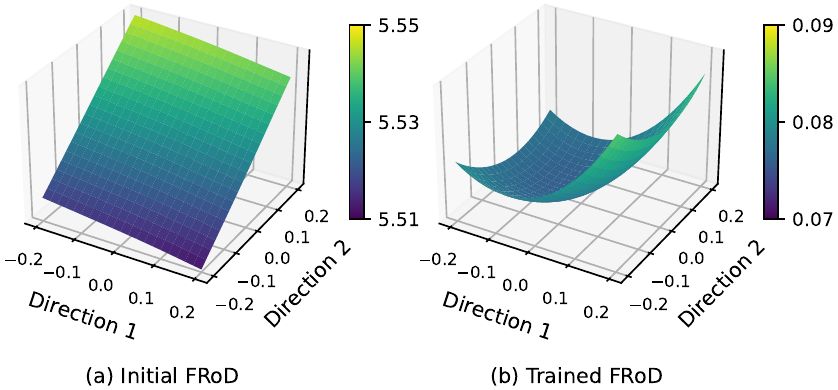}
\caption{Visualization of Loss Landscapes for FRoD in Model Parameter Space.}
\label{fig:frod_landscape}
\end{figure}

\begin{figure*}[t]
\centering
\includegraphics[width=\linewidth]{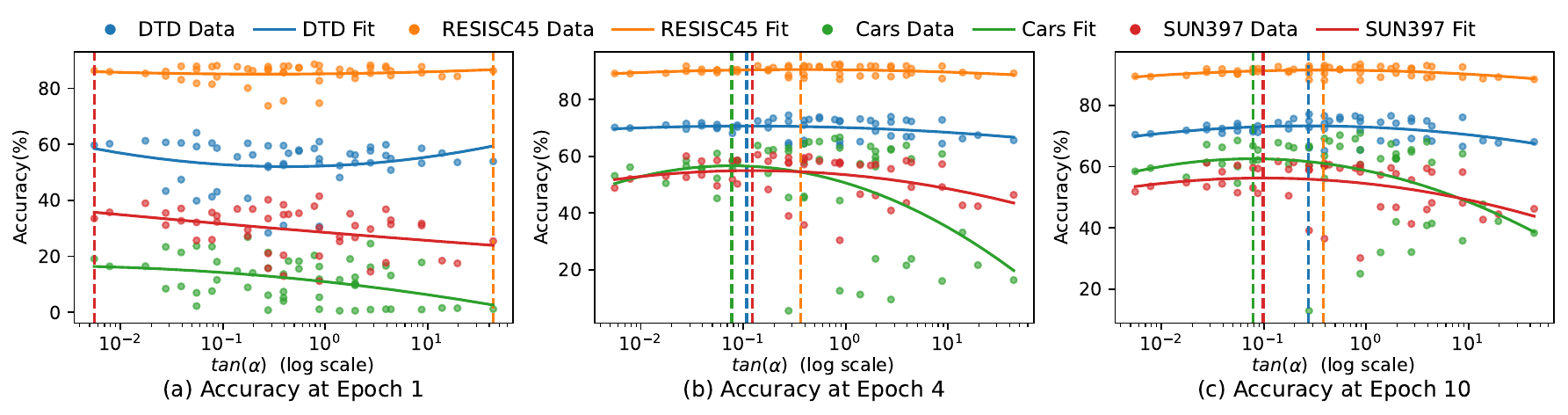}
% \caption{Impact of Rotation Angle $\alpha$ on Classification Accuracy. Three aligned panels chart the scatter and fitted trends showing how varying the off-axis rotation angle and its magnitude systematically influences accuracy on four vision datasets at Epochs 1, 4, and 10.}
\caption{Impact of off-axis rotation angle $\alpha$ on classification accuracy.  
Scatter plots and fitted trends illustrate how varying the rotation angle and its magnitude affect accuracy at epochs 1, 4, and 10 on four vision datasets.}
\label{fig:frod_ablation}

\end{figure*}

% Figure~\ref{fig:frod_landscape} provides a focused visualization of FRoD’s loss landscape before and after training. As quantified in Appendix, our hierarchical joint decomposition , reconstructing each layer with an error below 1e-5, thereby preserving the full generalized singular spectrum with smooth initial loss landscape as Figure~\ref{fig:frod_landscape}a rather than only the largest components selected by PiSSA initialization(Figure~\ref{fig:motivation}c). Compared with VeRA’s initial curvature (Figure~\ref{fig:motivation}d), our method exhibits a same isotropic and flatter loss surface at initialization (Figure~\ref{fig:frod_landscape}a). This smoother geometry supports the use of larger learning rates without introducing instability.
% After convergence, FRoD (Figure~\ref{fig:frod_landscape}b) maintains a convex and well-conditioned loss landscape, closely resembling that of Full Fine-Tuning (Figure~\ref{fig:motivation}e), yet with significantly fewer trainable parameters. This highlights FRoD’s ability to preserve optimization stability throughout training while benefiting from a full-rank update space, in contrast to VeRA’s post-training curvature distortion (Figure~\ref{fig:motivation}h) caused by limited parameter degrees of freedom.

% Together, these observations confirm that hierarchical joint decomposition not only enables better initial geometry but also ensures robust convergence dynamics under aggressive learning schedules.

To elucidate why FRoD achieves superior convergence, we visualize its loss landscape in \figurename~\ref{fig:frod_landscape}. As quantified in the Appendix D.5, our hierarchical joint decomposition reconstructs each layer with a maximum error below 1e-5, effectively preserving the complete generalized singular-value spectrum. This initialization results in a smooth and isotropic initial loss surface (\figurename~\ref{fig:frod_landscape}a). In contrast, PiSSA’s reliance on dominant components leads to a more fragmented initial geometry (\figurename~\ref{fig:motivation}c), while VeRA exhibits higher initial curvature (\figurename~\ref{fig:motivation}d).

% After training, FRoD (\figurename~\ref{fig:frod_landscape}b) retains a convex, well-conditioned loss landscape that closely mirrors full fine-tuning (Figure~\ref{fig:motivation}e) while requiring fewer trainable parameters. This outcome underscores FRoD’s capacity to maintain optimization stability via RDoF, whereas VeRA’s limited degrees of freedom induce pronounced curvature distortion after training (\figurename~\ref{fig:motivation}h).
After optimization, FRoD (\figurename~\ref{fig:frod_landscape}b) maintains a well-conditioned, convex landscape that closely resembles full fine-tuning (\figurename~\ref{fig:motivation}e), despite using significantly fewer trainable parameters. This stability confirms that the rotational degrees of freedom in FRoD provide sufficient flexibility for full-rank updates without inducing the pronounced curvature distortion seen in VeRA (\figurename~\ref{fig:motivation}h). Ultimately, these geometric advantages allow FRoD to support aggressive learning rates and achieve faster, more robust convergence across diverse tasks. Further experiments on additional initialization methods are provided in Appendix D.6.

\begin{table}[t]
\centering
\begin{tabular}{c c c | c c c}
\toprule
\multicolumn{3}{c|}{\textbf{Hyperparameters}} & \multicolumn{3}{c}{\textbf{Avg.\ Acc.(Epoch) (\%)}} \\
\cmidrule(lr){1-3}\cmidrule(lr){4-6}
$s$ & $\text{lr}(S)$ & $\text{lr}(\Sigma_i)$ & 1 & 4 & 10 \\
\midrule

$0.0$ & 0 & 1e-04 & 25.89 & 55.89 & 70.77* \\
% \rowcolor{red!7}
$0.0$ & 0 & 5e-04 & 34.20 & 66.52 & \textbf{71.26} \\
$0.0$ & 0 & 1e-03 & 42.13 & \textbf{68.58} & 70.86 \\
$0.0$ & 0 & 5e-03 & \textbf{50.03} & 65.66 & 67.28 \\
\midrule
% $0.01$ & 1e-05 & 0 & 29.83 & 62.24 & 71.17 \\
% $0.01$ & 5e-05 & 0 & 44.38 & 68.71 & 71.49 \\
$0.01$ & 1e-04 & 0 & 47.72 & \textbf{71.33} & 72.88 \\
% $0.01$ & 5e-04 & 0 & 42.29 & 53.69 & 59.18 \\
$0.02$ & 1e-05 & 0 & 29.96 & 62.47 & 70.77 \\
$0.02$ & 5e-05 & 0 & 44.51 & 68.87 & 70.92 \\
% \rowcolor{red!7}
$0.02$ & 1e-04 & 0 & \textbf{49.62} & 70.89 & \textbf{73.12} \\
$0.02$ & 5e-04 & 0 & 41.59 & 55.28 & 60.60 \\
% $0.1$ & 1e-05 & 0 & 29.89 & 61.65 & 70.97 \\
% $0.1$ & 5e-05 & 0 & 43.62 & 68.54 & 71.35 \\
$0.1$ & 1e-04 & 0 & 46.95 & 70.68 & 72.77 \\
% $0.1$ & 5e-04 & 0 & 40.64 & 56.43 & 61.20 \\
\midrule
% \rowcolor{gray!20}
$0.1$ & 5e-04 & 5e-03 & 37.47 & 48.51 & 52.22 \\
% $0.01$ & 1e-04 & 5e-04 & 49.39 & 71.35 & 73.91 \\
% $0.1$ & 5e-05 & 1e-03 & 49.57 & \textbf{72.37} & 74.09 \\
% \rowcolor{red!15}
% $0.01$ & 5e-05 & 1e-03 & 48.78 & 72.17 & 74.16 \\
$0.02$ & 1e-04 & 1e-03 & 48.36 & \textbf{71.85} & \textbf{74.59} \\
\bottomrule
\end{tabular}
\caption{Hyperparameter sensitivity of FRoD on four vision datasets. 
The table reports average accuracy (\%) at epochs 1, 4, and 10 for combinations of sparsity $s \in \{ 0.01, 0.02, 0.1\}$, off-axis learning rate $lr(S) \in \{\text{1e-5}, \text{5e-5}, \text{1e-4}, \text{5e-4}\}$, and on-axis learning rate $lr(\Sigma_i) \in \{ \text{1e-4}, \text{5e-4}, \text{1e-3}, \text{5e-3}\}$.  Rows are shaded to indicate the best (pink), second best (light pink), and worst (light gray) hyperparameter settings. An asterisk (*) denotes configurations that fail to converge on most datasets within 10 epochs.}
% \caption{
% Hyperparameter Sensitivity of FRoD on Four Vision Datasets. 
% Average accuracy (\%) at epochs 1, 4, and 10 is reported for various sparsity and learning rate settings. Row colors indicate best (pink), second-best (light pink), worst (gray). An asterisk (*) indicates non-convergence within 10 epochs.
% }
\label{tab:ablation_study}
\end{table}

\subsection{Ablation Study}

To better understand the roles of on-axis and off-axis updates, we conduct a detailed ablation study, separately disabling each component in FRoD. 

\noindent\textbf{Impact of On-axis.} When removing the off-axis updates, we observe that increasing the learning rate for the dominant singular directions ($\Sigma$) significantly accelerates early-stage convergence—as reflected in high accuracy at epoch 1, as \tablename~\ref{tab:ablation_study}. However, overly large values introduce instability and lead to suboptimal final performance by epoch 10. This reveals a clear trade-off between convergence speed and stability for the on-axis learning rate.

\noindent\textbf{Impact of Off-axis.} In contrast, when removing the on-axis component, the off-axis updates alone offer higher training stability due to their broader parameter search space. Yet, large off-axis learning rates do not guarantee better performance, as these directions are not strictly orthogonal and can easily distort the update. Consequently, smaller learning rates are preferred for off-axis updates to preserve stability.

\noindent\textbf{Impact of learning rate.} Importantly, none of the ablated configurations outperforms the jointly activated setting, which combines both components and achieves the highest accuracy. 
Nonetheless, the overall performance remains sensitive to different learning rate combinations. Some poorly tuned configurations even underperform the ablated variants.

To investigate this further, we approximate, using learning rate, the rotation angle $\alpha$ between the on-axis and off-axis components, defined by Eq. 14.
As shown in Figure 6, Early-stage performance (epoch 1) benefits from more aggressive updates in either axis. However, the best final accuracy (epoch 10) emerges when $\tan \alpha \in [0.05, 0.2]$, validating our theoretical insight that moderate rotation preserves strength alignment as Eq.17. Appendix E provides implementation details.
This analysis confirms that well-calibrated rotational degrees of freedom not only enhance early convergence but also maintain representation stability throughout training.
% RX TODO: 有无一些实验在Appendix？ 是否需要说明突出？

\section{Conclusion}

We present FRoD, a novel PEFT approach that synergistically integrates hierarchical orthogonal decomposition and sparse perturbations to achieve both rapid convergence and high expressiveness. 
% In this paper, we introduced FRoD, a novel parameter-efficient fine-tuning method that combines joint orthogonal decomposition with sparse perturbations to enhance both convergence speed and expressiveness. 
By extracting shared latent bases across layers and injecting rotational degrees of freedom via sparse matrices, FRoD substantially expands the effective update space at minimal parameter cost. 
% Our hierarchical decomposition extracts a shared latent basis across layers, while the sparse perturbation matrix provides rotational degrees of freedom, expanding the update space without significant parameter overhead. 
Empirical results on 20 tasks across vision, reasoning, and language understanding demonstrate that 
FRoD consistently achieves or surpasses full fine-tuning and leading PEFT baselines, converging within only 1--4 epochs and exhibiting strong robustness to random initialization. 
% FRoD converges in only 1--3 epochs for visual tasks and 2--4 epochs for other benchmarks, achieving or surpassing full fine-tuning performance and outperforming state-of-the-art PEFT methods. 
These findings highlight FRoD’s capability to reconcile parameter efficiency with adaptable optimization dynamics. 
% The proposed method is robust to random initialization and achieves high accuracy with rapid convergence, requiring minimal extra parameters. 
As future work, we will explore leveraging the high-rank sparse subspaces induced by FRoD for continual learning and model merging, potentially enabling more efficient knowledge integration and mitigating catastrophic forgetting.
% In future work, we aim to leverage the high-rank sparse orthogonal subspace enabled by FRoD to develop novel continual learning and model merging strategies, facilitating efficient knowledge integration and mitigating catastrophic forgetting.

\section{Acknowledgements}

This work was supported by the National Natural Science Foundation of China (62302022, 62202029), and Young Elite Scientists Sponsorship Program by CAST (NO.20230NRC001).

% Keep a single bibliography database for the whole paper
% \bibliography{aaai2026}

\putbib[reference]
\end{bibunit}
\onecolumn
\begin{bibunit}[aaai2026]
\begin{center}
    {\LARGE \textbf{Appendix}}\\[2mm]
    \vspace{0.5em}
    % \large (FRoD: Full-Rank Efficient Fine-Tuning with Rotational Degrees for 
    % Fast Convergence)
\end{center}

\vspace{1em}

\section*{A. Convergence Analysis}
Adam\cite{kingma_adam_2017} has become the most widely used first-order optimizer in modern deep-learning practice. Rather than explicitly computing the Hessian matrix\cite{Nocedal_optimiztion_2006}, it adapts the learning rate in a sophisticated, data-driven way, thereby partially emulating the behavior of second-order optimization methods. The second-order Taylor expansion of the loss $L(\theta)$ around initial parameters $\theta_0$ is:
\begin{equation}
L(\theta) \approx L(\theta_0) + \nabla L(\theta_0)^\top (\theta - \theta_0) + \frac{1}{2} (\theta - \theta_0)^\top H(\theta_0) (\theta - \theta_0),
\end{equation}
where$ \Delta L(\theta_0) $represents the first-order moment and $H(\theta_0) = \frac{\partial^2 L}{\partial \theta \partial \theta^\top}$ is the Hessian matrix: 
\begin{equation}
    H(\theta_0) = \begin{bmatrix}
\frac{\partial^2 L}{\partial \text{vec}(A) \partial \text{vec}(A)^\top} & \frac{\partial^2 L}{\partial \text{vec}(A) \partial \text{vec}(B)^\top} \\
\frac{\partial^2 L}{\partial \text{vec}(B) \partial \text{vec}(A)^\top} & \frac{\partial^2 L}{\partial \text{vec}(B) \partial \text{vec}(B)^\top}
\end{bmatrix}.
\end{equation}

Existing studies on methods\cite{meng_pissa_2024}, which initialize LoRA matrices from singular-value components, have focused almost exclusively on first-order statistics. Because a zero-initialized LoRA matrix starts with extremely small values, those works propose a more principled approach that initializes the LoRA matrix using the factor matrices obtained from singular-value decomposition. However, the Hessian matrix also influences the convergence behavior of these Parameter-Efficienct Fine-Tuning (PEFT) methods.
Analyzing the Hessian matrix enables us to characterize the gradient-descent dynamics of Full Fine-Tuning (Full FT), LoRA\cite{hu_lora_2021}, PiSSA\cite{meng_pissa_2024}, and VeRA\cite{meng_pissa_2024}, thereby comparing the convergence rates of these methods. 
Some previous work \cite{zhang_why_2024, tabeart_new_2022} analyze the performance of gradient-descent algorithms by examining the condition number $\tau$ of the Hessian matrix H:
\begin{equation}
    \tau = \frac{\max\|H\|}{\min \|H\|},
\end{equation}
where the norm $\|\|$ is typically taken to be the 2-norm, the ratio between the largest and smallest eigenvalues of $H$. However, under a low-rank adapter setting, the effective rank is restricted to $r \ll \min \{n,m\}$, implying that the smallest eigenvalue of $H$  is zero; consequently, the condition number $\tau$ becomes unbounded. To circumvent this problem, we add a small ridge term $\epsilon I$ to the Hessian and analyze the resulting optimization dynamics through the regularized condition number
\begin{equation}
    \dot{\tau} \triangleq \frac{\lambda_{\max}(H+\epsilon I)}{\lambda_{\min}(H+\epsilon I)},
\end{equation}
which remains finite for any $\epsilon>0$.
% This appendix A provides a detailed derivation of the second-order Taylor expansion for PEFT, focusing on the Hessian matrices to analyze optimization dynamics.

% % Second-Order Taylor Expansion
% The second-order Taylor expansion of the loss $L(\theta)$ around initial parameters $\theta_0$ is:
% \begin{equation}
% L(\theta) \approx L(\theta_0) + \nabla L(\theta_0)^\top (\theta - \theta_0) + \frac{1}{2} (\theta - \theta_0)^\top H(\theta_0) (\theta - \theta_0),
% \end{equation}
% where $H(\theta_0) = \frac{\partial^2 L}{\partial \theta \partial \theta^\top}$ is the Hessian matrix:

% \begin{equation}
%     H(\theta_0) = \begin{bmatrix}
% \frac{\partial^2 L}{\partial \text{vec}(A) \partial \text{vec}(A)^\top} & \frac{\partial^2 L}{\partial \text{vec}(A) \partial \text{vec}(B)^\top} \\
% \frac{\partial^2 L}{\partial \text{vec}(B) \partial \text{vec}(A)^\top} & \frac{\partial^2 L}{\partial \text{vec}(B) \partial \text{vec}(B)^\top}
% \end{bmatrix}.
% \end{equation}

\subsection*{A.1. Hessian Analysis of LoRA}

\textbf{Parameterization}: LoRA gets $W = W_0 +  AB$, with $A \in \mathbb{R}^{m \times r}$, $B \in \mathbb{R}^{r \times n}$, where $\theta = [\text{vec}(A); \text{vec}(B)]$.

\textbf{Gradient}:
\begin{equation}
\frac{\partial L}{\partial A_{ij}} =  \sum_l \frac{\partial L}{\partial W_{\text{adapted},il}} B_{jl}, \quad \frac{\partial L}{\partial B_{ij}} =  \sum_k A_{ki} \frac{\partial L}{\partial W_{\text{adapted},kj}},
\end{equation}
\begin{equation}
\nabla L(\theta_0) = \begin{bmatrix}
\text{vec}\left(  \left( \frac{\partial L}{\partial W_{\text{adapted}}} \right) B_0^\top \right) \\
\text{vec}\left(  A_0^\top \left( \frac{\partial L}{\partial W_{\text{adapted}}} \right) \right)
\end{bmatrix}.
\end{equation}

\textbf{Hessian}:
\begin{equation}
\frac{\partial^2 L}{\partial A_{ij} \partial A_{pq}} = ^2 \sum_{l,t} \frac{\partial^2 L}{\partial W_{\text{adapted},il} \partial W_{\text{adapted},pt}} B_{jl} B_{qt} = ^2 (B B^\top)_{jl} \delta_{ip} H_W,
\end{equation}
\begin{equation}
\frac{\partial^2 L}{\partial \text{vec}(A) \partial \text{vec}(A)^\top} = ^2 (B B^\top \otimes H_W),
\end{equation}
\begin{equation}
\frac{\partial^2 L}{\partial B_{ij} \partial B_{pq}} = ^2 \sum_{k,s} A_{ki} A_{sp} \frac{\partial^2 L}{\partial W_{\text{adapted},kj} \partial W_{\text{adapted},sq}},
\end{equation}
\begin{equation}
\frac{\partial^2 L}{\partial \text{vec}(B) \partial \text{vec}(B)^\top} = ^2 (A^\top A \otimes H_W).
\end{equation}
Assuming $H_W \approx \lambda I_{mn}$:
\begin{equation}
H(\theta_0) \approx \lambda \begin{bmatrix}
\|B_0\|_F^2 I_{mr} & 0 \\
0 & \|A_0\|_F^2 I_{rn}
\end{bmatrix}.
\end{equation}
The condition number $\dot \tau$ is $(\max(\|B_0\|_F^2, \|A_0\|_F^2) + \epsilon) / (\epsilon)$.When the parameters are initialized with $A \sim \mathcal{N}(0,\sigma^{2})$ and $B = 0$, the regularized condition number
\begin{equation}
    \dot{\tau} \;=\; \frac{\epsilon + \|A_{0}\|{F}^{2}}{\epsilon}
\end{equation}
is close to 1, so the corresponding loss landscape is relatively flat (\figurename~\ref{fig:landscape}b). After some training, $\|A{0}\|{F}^{2}$ and $\|B{0}\|_{F}^{2}$ both increase, $\dot{\tau}$ grows larger, and the landscape becomes much steeper (\figurename~\ref{fig:landscape}f).

\subsection*{A.2. Hessian Analysis of PiSSA}

\textbf{Parameterization}: Similar to LoRA, PiSSA parameterizes the weight matrix as $W = W_0 + AB$, where $A_0 = U_{[:, :r]} S_{[:r, :r]}^{1/2}$ and $B_0 = S_{[:r, :r]}^{1/2} V_{[:, :r]}^\top$.

\textbf{Gradient}:
\begin{equation}
\nabla L(\theta_0) = \begin{bmatrix}
\text{vec}\left( \left( \frac{\partial L}{\partial W_{\text{adapted}}} \right) B_0^\top \right) \\
\text{vec}\left( A_0^\top \left( \frac{\partial L}{\partial W_{\text{adapted}}} \right) \right)
\end{bmatrix}.
\end{equation}

\textbf{Hessian}: Identical to LoRA:
\begin{equation}
H(\theta_0) \approx \lambda \left( \sum_{i=1}^r \sigma_i \right) \begin{bmatrix}
I_{mr} & 0 \\
0 & I_{rn}
\end{bmatrix}.
\end{equation}
Here, $\sigma_i$ denotes the strength (magnitude) of the dominant singular values that appear on the diagonal of the matrix $S_{[:r,:r]}$. For PiSSA, the initial condition number
\begin{equation}
    \dot{\tau} \;=\; \frac{\sum_{i=1}^{r}\sigma_{i} + \epsilon}{\epsilon},
\end{equation}
implies that the regularized Hessian is already highly ill-conditioned at initialization, which makes the loss landscape in \figurename~\ref{fig:landscape}c extremely rugged.
As training proceeds, however, $\dot{\tau}$ gradually decreases (see \figurename~\ref{fig:landscape}g), and in the ideal case it reaches the same value as LoRA once both methods converge to the same optimum.

\begin{figure}[t]
\centering
\includegraphics[width=\linewidth]{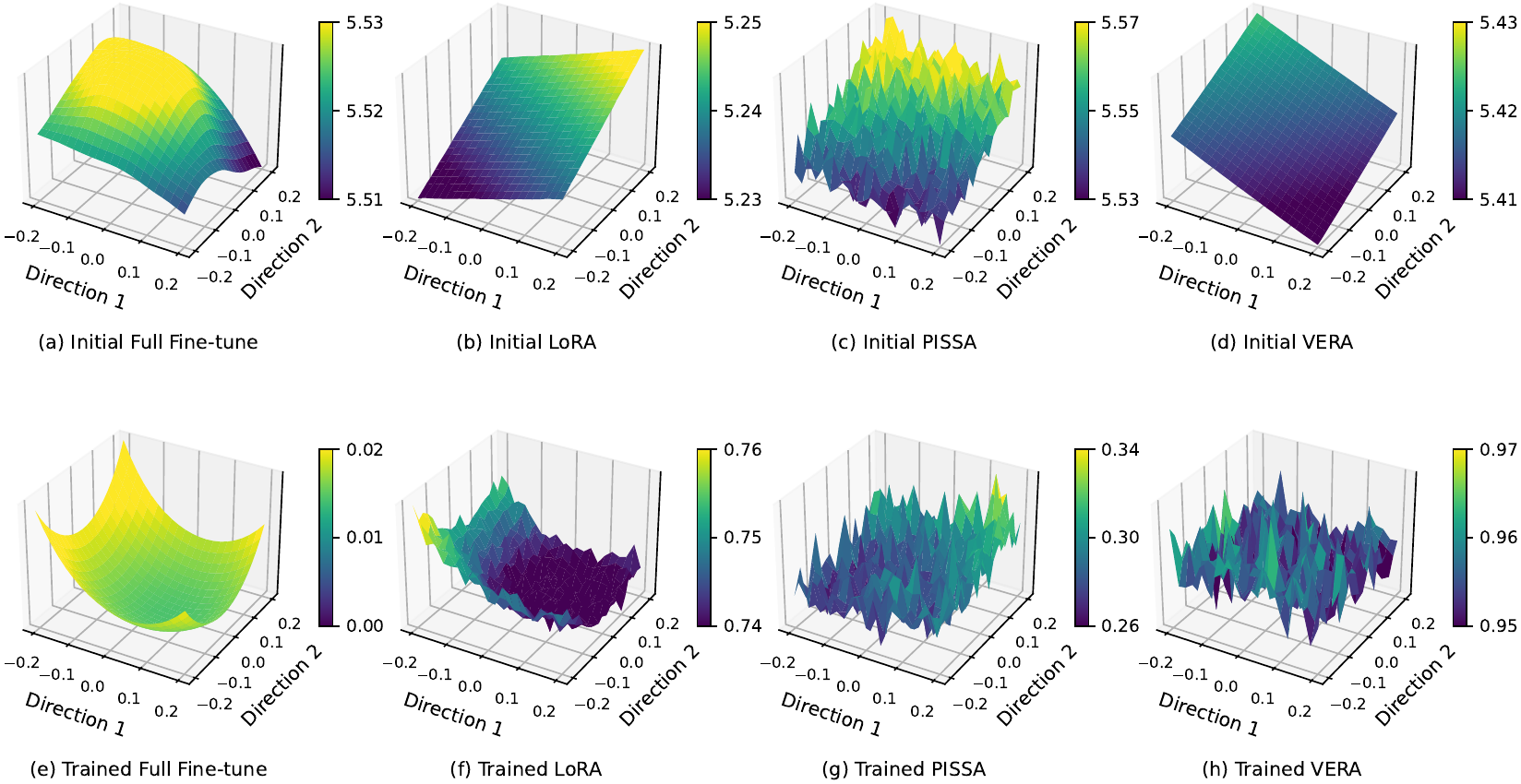}
\caption{Comparison of loss landscapes for four representative fine-tuning methods in a principal parameter subspace.}
\label{fig:landscape}
\end{figure}

\subsection*{A.3. Hessian Analysis of VeRA}

\textbf{Parameterization}: VeRA initailizes 2 shared random-basis, $W = W_0 + \Lambda_b B \Lambda_d A$, with $\Lambda_b = \text{diag}(b)$, $\Lambda_d = \text{diag}(d)$, $b, d \in \mathbb{R}^r$, and $\theta = [b; d]$.

\textbf{Gradient}:
\begin{equation}
\frac{\partial L}{\partial b_k} = d_k B_{kk} \langle \frac{\partial L}{\partial W_{\text{adapted}}}, A_{:k} \rangle, \quad \frac{\partial L}{\partial d_k} = b_k B_{kk} \langle \frac{\partial L}{\partial W_{\text{adapted}}}, A_{:k} \rangle.
\end{equation}

\textbf{Hessian}:
\begin{equation}
\frac{\partial^2 L}{\partial b_k \partial b_p} = \lambda d_k d_p B_{kk} B_{pp} (A^\top A)_{kp}, \quad \frac{\partial^2 L}{\partial d_k \partial d_p} = \lambda b_k b_p B_{kk} B_{pp} (A^\top A)_{kp},
\end{equation}
\begin{equation}
\frac{\partial^2 L}{\partial b_k \partial d_p} \approx \lambda B_{kk} (A^\top \frac{\partial L}{\partial W_{\text{adapted}}})_{kp} \delta_{kp}.
\end{equation}
The Hessian is approximately:
\begin{equation}
H(\theta_0) \approx \lambda \begin{bmatrix}
\text{diag}(d_0^2 \circ B_{0,\text{diag}}^2 \circ \text{diag}(A_0^\top A_0)) & \text{diag}(B_{0,\text{diag}} \circ \text{diag}(A_0^\top \frac{\partial L}{\partial W_{\text{adapted}}})) \\
\text{diag}(B_{0,\text{diag}} \circ \text{diag}(A_0^\top \frac{\partial L}{\partial W_{\text{adapted}}})) & \text{diag}(b_0^2 \circ B_{0,\text{diag}}^2 \circ \text{diag}(A_0^\top A_0))
\end{bmatrix}.
\end{equation}
For VeRA, a comparatively high-rank fine-tuning scheme, the regularized Hessian exhibits the same well-behaved conditioning at initialization as in LoRA, especially because the b vector is initialized to zero. Consequently, VeRA’s initial regularized condition number is close to 1, indicating an almost planar loss landscape, as \figurename~\ref{fig:landscape}d. This near-isotropy allows VeRA to operate with an unusually large learning rate (e.g. 1e-2) in practical fine-tuning scenarios, a rate that even full-parameter fine-tuning typically cannot sustain.

% Evaluation
\subsection*{A.4. Brife Summary}

A Hessian-based comparison of LoRA, PiSSA, and VeRA reveals that convergence speed depends on two competing factors: the magnitude of first-order gradients and the spectral condition number. LoRA, initialized with $B_0 = 0$, enjoys a regularized condition number $\dot{\tau}!\approx!1$ and an almost isotropic landscape, yet its vanishing first-order gradient moments slow early‐stage optimization. PiSSA amplifies these moments by seeding the adapter with SVD factors, but the direct injection of large singular values into a low-rank module inflates the initial $\dot{\tau}$, ruggedizes the landscape, and necessitates conservative learning rates. VeRA preserves a LoRA-like flat landscape through shared diagonal bases $\Lambda_b$ and $\Lambda_d$, which permits learning rates as high as $10^{-2}$; however, its zero-initialized strength vectors $(b,d)$ compress the gradient magnitude to near zero. Addressing either “tiny gradients” or “spectral ill-conditioning” in isolation inevitably compromises the other, highlighting the need for a structural compromise.

We therefore introduce Hierarchical Joint Decomposition (HJD). The method retains VeRA’s layer-agnostic random bases to keep the global Hessian nearly isotropic, but performs a joint factorization of higher-order statistics—such as variance and block-diagonal second-moment tensors—across all layers, and assigns each layer a data-aware (non-zero) strength vector. Consequently, HJD provides both realistically scaled gradients and a flat landscape ($\dot{\tau}!\approx!1$), thereby enabling aggressive learning rates for rapid convergence while sidestepping PiSSA’s initial spectral imbalance and LoRA’s gradient starvation. Its global–local, two-level coupling further allows each layer to adapt its singular-value spectrum during fine-tuning without violating overall isotropy, achieving parameter efficiency and stable optimization simultaneously.

\section*{B. Proof of Parameter Degree of Freedom}
% Defining the theorem and proof for PEFT update subspaces

We characterize each Parameter-Efficient Fine-Tuning (PEFT) method by the \emph{update subspace} $\mathbb{U} \subset \mathbb{R}^{m \times n}$ it can explore during training. A wider subspace affords greater expressivity, whereas a narrow one restricts adaptation.

% \begin{theorem}
Let $\mathbb{R}^{m \times n}$ denote the space of $m \times n$ real matrices, representing the weight matrices of a neural network layer. For a PEFT method, the update subspace $\mathbb{U} \subset \mathbb{R}^{m \times n}$ defines the set of possible update matrices that can be applied to the weight matrix during training. For the Low-Rank Adaptation (LoRA) and Vector-based Rank Adaptation (VeRA) methods, the update subspaces and their dimensions are characterized as follows:

\begin{enumerate}
    \item \textbf{LoRA Update Subspace}:
    \begin{equation}
    \mathbb{U}_{\text{LoRA}} = \left\{ BA \mid B \in \mathbb{R}^{m \times r}, A \in \mathbb{R}^{r \times n} \right\},
    \end{equation}
    with dimension:
    \begin{equation}
    \dim(\mathbb{U}_{\text{LoRA}}) = r(m + n - r).
    \end{equation}
    \item \textbf{VeRA Update Subspace}:
    \begin{equation}
    \mathbb{U}_{\text{VeRA}} = \operatorname{span} \left\{ B_{:k} A_{:k}^\top \right\}_{k=1}^r,
    \end{equation}
    where $B_{:k} \in \mathbb{R}^m$ and $A_{:k} \in \mathbb{R}^n$ are the $k$-th columns of matrices $B \in \mathbb{R}^{m \times r}$ and $A \in \mathbb{R}^{n \times r}$, respectively, and with dimension:
    \begin{equation}
    \dim(\mathbb{U}_{\text{VeRA}}) = r + n.
    \end{equation}
\end{enumerate}
Here, $r$ is the rank parameter, satisfying $1 \leq r \leq \min(m, n)$.
% \end{theorem}

% Providing the proof
\subsection*{B.1. Proof}

\subsubsection*{B.1.1. LoRA Update Subspace and Dimension}
\textbf{Update Subspace Characterization}:
For LoRA, the update matrix is of the form $BA$, where $B \in \mathbb{R}^{m \times r}$ and $A \in \mathbb{R}^{r \times n}$. Thus:
\begin{equation}
\mathbb{U}_{\text{LoRA}} = \left\{ BA \mid B \in \mathbb{R}^{m \times r}, A \in \mathbb{R}^{r \times n} \right\}.
\end{equation}

This set consists of all $m \times n$ matrices expressible as the product of a matrix $B$ of size $m \times r$ and a matrix $A$ of size $r \times n$. Since $BA \in \mathbb{R}^{m \times n}$, $\mathbb{U}_{\text{LoRA}}$ is a subset of $\mathbb{R}^{m \times n}$.

\textbf{Dimension Calculation}:
To compute $\dim(\mathbb{U}_{\text{LoRA}})$, we determine the number of free parameters in $BA$, accounting for redundancies.

\textbf{Parameter Count}: The matrix $B \in \mathbb{R}^{m \times r}$ has $m \cdot r$ entries, and $A \in \mathbb{R}^{r \times n}$ has $r \cdot n$ entries. The total number of parameters is:
\begin{equation}
m \cdot r + r \cdot n = r(m + n).
\end{equation}
    
\textbf{Redundancy}: The factorization $BA$ is not unique since for any invertible matrix $Q \in \mathbb{R}^{r \times r}$, we have:
\begin{equation}
BA = (B Q)(Q^{-1} A).
\end{equation}

The group of invertible matrices $GL(r, \mathbb{R})$ has dimension $r^2$, as an $r \times r$ matrix has $r^2$ entries, and invertibility imposes an open condition (non-zero determinant) that does not reduce the dimension.

\textbf{Effective Dimension}: The parameter space of $(B, A)$ is $\mathbb{R}^{m \times r} \times \mathbb{R}^{r \times n}$, with dimension $r(m + n)$. Subtracting the redundancy dimension $r^2$, the dimension of $\mathbb{U}_{\text{LoRA}}$ is:
\begin{equation}
r(m + n) - r^2 = r(m + n - r).
\end{equation}

Thus:
\begin{equation}
\dim(\mathbb{U}_{\text{LoRA}}) = r(m + n - r).
\end{equation}

\subsubsection*{B.1.2. VeRA Update Subspace and Dimension}
\textbf{Update Subspace Characterization}:
For VeRA, the update matrix is a linear combination of rank-1 matrices:
\begin{equation}
\Delta W =  \Lambda_b B \Lambda_d A = \sum_{k=1}^r \alpha_k B_{:k} d_k A_{:k}^\top,
\end{equation}
where $B_{:k} \in \mathbb{R}^m$ and $A_{:k} \in \mathbb{R}^n$ are the $k$-th columns of $B \in \mathbb{R}^{m \times r}$ and $A \in \mathbb{R}^{n \times r}$, respectively, and $\alpha_k \in \mathbb{R}$ are coefficients. Thus:
\begin{equation}
\mathbb{U}_{\text{VeRA}} = \operatorname{span} \left\{ B_{:k} A_{:k}^\top \right\}_{k=1}^r.
\end{equation}

\textbf{Dimension Calculation}:
To compute $\dim(\mathbb{U}_{\text{VeRA}})$, we analyze the span parameterized by $B$, $A$, and coefficients $\alpha_k$.

Assume $B \in \mathbb{R}^{m \times r}$ is fixed (a common setup in VeRA to reduce parameters), and we optimize over $A \in \mathbb{R}^{n \times r}$. The matrix $A$ has $n \cdot r$ entries. The update matrix is:
\begin{equation}
\sum_{k=1}^r \alpha_k B_{:k} d A_{:k}^\top.
\end{equation}

The matrix $B_{:k} A_{:k}^\top$ vectorizes to $A_{:k} \otimes B_{:k} \in \mathbb{R}^{m \cdot n}$. The subspace is:
\begin{equation}
\mathbb{U}_{\text{VeRA}} = \operatorname{span} \left\{ A_{:k} \otimes B_{:k} \right\}_{k=1}^r.
\end{equation}

\textbf{Dimension Analysis}: If $B_{:k}$ are fixed and linearly independent, and $A_{:k}$ are parameterized by magnitude vector $d$, we consider a simplified case where $d$ is a single vector $v \in \mathbb{R}^n$ scaled across $r$ terms, or $A$ has a constrained form. The provided dimension $r + n$ suggests:
    
\begin{itemize}
    \item A vector $d \in \mathbb{R}^n$ contributing $n$ parameters.
    \item Coefficients $\{\alpha_k\}_{k=1}^r$ contributing $r$ parameters.
\end{itemize}

Thus:
\begin{equation}
    \dim(\mathbb{U}_{\text{VeRA}}) = r + n.
\end{equation}

This assumes a parameterization where $A, B$ is fixed, and $d$ contributes $n$ parameters plus $r$ coefficients, matching the given dimension.

% Concluding remarks
\subsection*{B.2. Brief Summary}

Parameter Degree of Freedom analysis exposes VeRA’s inherent weakness in its update space: its degrees of freedom, $\dim(\mathbb{U}_{\text{VeRA}})=r+n$, scale only linearly with LoRA, yet are orders of magnitude smaller than the $mn$ available to Full FT. This overly narrow subspace forces the gradient energy during late training to collapse into only a few dominant singular directions, leaving almost all others effectively frozen. Consequently, only a handful of large singular values persist; the Hessian spectrum becomes severely imbalanced, and VeRA exhibits the most rugged loss landscape among the four methods (see \figurename~\ref{fig:landscape}h), thereby hampering—if not destabilising—late-stage convergence.

To alleviate this bottleneck, we introduce sparse off-axis rotational degrees of freedom into the update space. Without altering the magnitude information carried by the shared diagonal matrix, we inject a small set of sparse rotational components that impart the shared matrix with a subtle yet crucial rotation angle. This extra freedom (i) expands the model’s expressive power and (ii) softens the Hessian spectrum by shrinking extreme eigenvalue gaps, thereby making the curvature distribution resemble the flatter profile of Full FT (see \figurename~\ref{fig:landscape}e). The result is a training process that retains fast convergence under large step sizes while simultaneously achieving greater stability during the later optimisation stages.

\section*{C. Proof of the FRoD Property}

The FRoD method is characterized by two fundamental components: a Hierarchical Joint Decomposition and an embedded rotational degree of freedom. In the following, we formally analyze these two components and derive the main theorems and corollaries that underpin the theoretical foundation of FRoD.

\subsection*{C.1. The Necessity of Hierarchical Joint Decomposition}

In parameter-efficient fine-tuning (PEFT), the introduction of shared projection matrices is critical for reducing memory overhead, especially when deploying models across multiple tasks or devices. Sharing a common basis across layers or tasks not only saves storage but also simplifies model merging and continual learning. Furthermore, from the perspective of matrix analysis, the orthogonality property is highly desirable due to its favorable numerical stability and algebraic tractability. In particular, orthogonal matrices preserve inner products and eigenvalue structure, making them an ideal choice for basis representations.

However, despite these advantages, we establish a key structural limitation in the form of the following non-existence result.

\begin{proposition}[Non-Existence for a Joint Decomposition with a Common Orthogonal Matrix]
Let ${A_1, A_2, \dots, A_k}$ be a set of real matrices. A joint factorization of the form
\begin{equation}
    A_i = U_i \Sigma_i V^T,
\end{equation}
where for each $i$, $U_i$ is an orthogonal matrix, $\Sigma_i$ is a diagonal matrix of singular values, and $V$ is a single \textbf{common} orthogonal matrix for all $i$, does not generally exist.
\end{proposition}

\begin{proof}
We prove this proposition by deriving a necessary condition for the existence of such a decomposition, and then demonstrating that this condition is not satisfied in general.

Assume the decomposition exists for each matrix $A_i \in \mathbb{R}^{m \times n}$ as:
\begin{equation}
A_i = U_i \Sigma_i V^T
\end{equation}
where $U_i \in \mathbb{R}^{m \times m}$ is orthogonal, $\Sigma_i \in \mathbb{R}^{m \times n}$ is diagonal with non-negative entries, and $V \in \mathbb{R}^{n \times n}$ is a common orthogonal matrix for all $i$.

Taking the transpose and computing $A_i^T A_i$:

\begin{align}
    A_i^T &= (U_i \Sigma_i V^T)^T = V \Sigma_i^T U_i^T \\
    A_i^T A_i &= V \Sigma_i^T U_i^T U_i \Sigma_i V^T = V (\Sigma_i^T \Sigma_i) V^T.
\end{align}

Define $D_i = \Sigma_i^T \Sigma_i$, which is symmetric and diagonal. Then:
\begin{equation} \label{eq:spectral}
A_i^T A_i = V D_i V^T.
\end{equation}

Equation \eqref{eq:spectral} represents the spectral decomposition of $A_i^T A_i$ with eigenvectors given by the common matrix $V$ and eigenvalues by $D_i$.

Hence, for all $i$, the matrices ${A_i^T A_i}$ must be simultaneously diagonalizable by the same orthogonal matrix $V$.

\begin{lemma}
A set of symmetric matrices can be simultaneously diagonalized by a single orthogonal matrix if and only if all matrices in the set commute, i.e.,
\begin{equation}
A_i^T A_i A_j^T A_j = A_j^T A_j A_i^T A_i, \quad \forall i, j
\end{equation}

\end{lemma}

This yields a necessary condition: the matrices ${A_i^T A_i}$ must commute. However, this condition is generally not satisfied unless the matrices ${A_i}$ are specially structured or aligned.

\noindent\textbf{Conclusion:} This counterexample shows that the required commutativity condition for joint spectral decomposition is not satisfied in general. Therefore, a factorization of the form $A_i = U_i \Sigma_i V^T$ with shared orthogonal $V$ does not exist for arbitrary ${A_i}$. This highlights the limitation of assuming a globally shared orthogonal basis in PEFT designs.

\end{proof}

As shown in Proposition~1, a joint decomposition with a common orthogonal matrix across arbitrary matrices does not generally exist, as the Gram matrices of arbitrary weights are generally non-commutative. This limitation motivates a relaxed structural constraint that still facilitates effective parameter sharing.

To this end, we adopt a hierarchical joint decomposition (HO-GSVD) to factorize each weight matrix \(W_i \in \mathbb{R}^{m \times n}\) as follows:
\begin{equation}
W_i \approx U_i \Sigma_i V^\top = (Q_i Z)\Sigma_i (R_{\text{global}}^\top Z)^\top,
\end{equation}
where:
\begin{itemize}
    \item \(U_i = Q_i Z \in \mathbb{R}^{m \times n}\) serves as an approximately orthogonal left basis;
    \item \(\Sigma_i \in \mathbb{R}^{n \times n}\) is a diagonal trainable strength matrix;
    \item \(V = R_{\text{global}}^\top Z \in \mathbb{R}^{n \times n}\) is a fixed shared latent basis;
    \item \(R_{\text{global}} \in \mathbb{R}^{n \times n}\) is an upper-triangular matrix from QR decomposition.
\end{itemize}

Here, \(R_{\text{global}}\) acts as a static projection operator into the shared latent space and remains frozen throughout training. We fine-tune the model by introducing a sparse perturbation \(S_i\) into the diagonal scaling matrix \(\Sigma_i\), leading to an updated weight:
\begin{equation}
W_i' = U_i (\Sigma_i + S_i) V^\top.
\end{equation}

The update matrix admits an orthogonal decomposition into two components:
\begin{itemize}
    \item \textbf{On-axis update} (magnitude adjustment):
    \begin{equation}
    \Delta W^{\text{on}}_i = U_i \Delta\Sigma_i Z^\top,
    \end{equation}
    \item \textbf{Off-axis rotation} (sparse perturbation):
    \begin{equation}
    \Delta W^{\text{off}}_i = U_i S_i Z^\top,
    \end{equation}
\end{itemize}

This formulation allows FRoD to introduce expressive rotational degrees of freedom within a stable shared basis, while preserving full-rank adaptation capability and promoting strong geometric regularity.

\subsection*{C.2. Theoretical Proof for the Preservation of the Singular Value Spectrum}

To theoretically justify the spectral robustness and geometric interpretability of our proposed update mechanism, we formalize and prove a series of results in the following subsections. Specifically, we begin by stating the following theorem:

\begin{theorem}[Spectral Stability Under Sparse Perturbations]
\label{thm:spectral_stability}
Let $\Sigma_i\in\mathbb{R}^{n\times n}$ be a diagonal matrix of singular values and let  
$S_i\in\mathbb{R}^{n\times n}$ be a perturbation such that  
\begin{equation}
\|S_i\|_0\ll n^2 
\quad\text{and}\quad
\max_{p,q}|s_{pq}|\le\varepsilon.
\end{equation}
Then, for every $k$,
\begin{equation}
\bigl|\sigma_k(\Sigma_i+S_i)-\sigma_k(\Sigma_i)\bigr|
\;\le\;
\|S_i\|_2 ,
\end{equation}
so the dominant singular directions encoded by $\Sigma_i$ are preserved up to a deviation bounded by $\|S_i\|_2$.
\end{theorem}

Then, we show that our update can be decomposed orthogonally with respect to the Frobenius inner product:

\begin{corollary}[Orthogonal Decomposition of Updates]
\label{cor:orthogonality}
Define the on-axis and off-axis updates as
\(
\Delta W_i^{\text{on}},\;\Delta W_i^{\text{off}}\in\mathbb{R}^{m\times n}.
\)
Then
\begin{equation}
\langle\Delta W_i^{\text{on}},\Delta W_i^{\text{off}}\rangle_F \;=\; 0 ,
\end{equation}
i.e.\ the two update components are orthogonal in the Frobenius inner product.
\end{corollary}

Finally, we derive an angular representation that characterizes the contribution of each component to the total update:

\begin{corollary}[Angular Representation of the Total Update]
\label{cor:angular_representation}
Let
\(
\tan\alpha=\|\Delta W_i^{\text{off}}\|_F/\|\Delta W_i^{\text{on}}\|_F
\)
and set
\(
\hat U_{\text{on}}=\Delta W_i^{\text{on}}/\|\Delta W_i^{\text{on}}\|_F,\;
\hat U_{\text{off}}=\Delta W_i^{\text{off}}/\|\Delta W_i^{\text{off}}\|_F.
\)
Then
\begin{equation}
\Delta W_i
=\|\Delta W_i\|_F\bigl(\cos\alpha\,\hat U_{\text{on}}
+\sin\alpha\,\hat U_{\text{off}}\bigr).
\end{equation}

For small $\alpha$ (when $\|\Delta W_i^{\text{on}}\|_F\gg\|\Delta W_i^{\text{off}}\|_F$),
\begin{equation}
\Delta W_i \approx \|\Delta W_i^{\text{on}}\|_F \left( \hat{U}_{\text{on}} + \alpha \cdot \hat{U}_{\text{off}} \right) R_\text{global}.
\end{equation}

Thus, the total update can be interpreted as a rotation from the axis-aligned (scaling) direction by a small angle $\alpha$, reflecting a controllable expansion as Shared Random-Basis PEFT. 
\end{corollary}

\subsubsection*{C.2.1. Proof of Theorem~\ref{thm:spectral_stability}: Spectral Stability Under Sparse Perturbations}

We consider a diagonal matrix of singular values $\Sigma_i \in \mathbb{R}^{n \times n}$ and a perturbation matrix $S_i \in \mathbb{R}^{n \times n}$ which is assumed to be sparse and entry-wise bounded. Our goal is to show that the singular values of the perturbed matrix $\Sigma_i + S_i$ remain close to those of $\Sigma_i$.

\begin{proof}
This result relies on Weyl’s inequality, a classical tool in spectral perturbation theory. For any two square matrices $A$ and $E$ of the same size, and for each singular value index $k$, Weyl’s inequality guarantees:
\begin{equation}
\bigl| \sigma_k(A + E) - \sigma_k(A) \bigr| \le \|E\|_2,
\end{equation}
where $|\cdot|_2$ denotes the spectral norm, i.e., the largest singular value of $E$.

In our setting, we take $A = \Sigma_i$ and $E = S_i$. Since $\Sigma_i$ is diagonal with real non-negative entries, its singular values are exactly its diagonal entries, and we assume they are ordered decreasingly. By applying Weyl’s inequality directly, we obtain:
\begin{equation}
\bigl| \sigma_k(\Sigma_i + S_i) - \sigma_k(\Sigma_i) \bigr| \le |S_i|_2 \quad \text{for all } k = 1, \dots, n.
\end{equation}

Now, to give a more informative bound, we estimate the spectral norm $|S_i|_2$ using its sparsity and element-wise bound. Recall that for any matrix $S$, we have the inequality:
\begin{equation}
\|S\|_2 \le \|S\|_F,
\end{equation}
where $\|S\|F$ is the Frobenius norm. Since $S_i$ has at most $s = \|S_i\|_0$ nonzero entries, and each entry satisfies $\|s_{pq}\| \le \varepsilon$, we get:
\begin{equation}
\|S_i\|F^2 = \sum{p,q} \|s_{pq}\|^2 \le s \cdot \varepsilon^2,
\end{equation}
and therefore:
\begin{equation}
\|S_i\|_2 \le \|S_i\|_F \le \sqrt{s} \cdot \varepsilon.
\end{equation}

Combining the inequalities, we conclude:
\begin{equation}
\bigl| \sigma_k(\Sigma_i + S_i) - \sigma_k(\Sigma_i) \bigr| \le \sqrt{\|S_i\|_0} \cdot \varepsilon,
\end{equation}
for every singular value index $k$. This result confirms that if the number of perturbed entries is small (i.e., $S_i$ is sparse), and each perturbation is small in magnitude (bounded by $\varepsilon$), then the entire spectrum of the matrix remains stable, with all singular values deviating by at most $\mathcal{O}(\sqrt{\|S_i\|_0} \cdot \varepsilon)$.

In particular, the dominant singular values, which control the main projection directions in low-rank approximations, are preserved up to a small bounded deviation. This justifies the use of a sparse additive matrix $S_i$ on top of $\Sigma_i$ in our method without significantly distorting the spectral structure.
\end{proof}

\subsubsection*{C.2.2. Proof of Corollary~\ref{cor:orthogonality}: Orthogonality of $\Delta W_i^{\text{on}}$ and$ \Delta W_i^{\text{off}}$}

We compute the Frobenius inner product between the on-axis and off-axis update components:
\begin{equation}
\langle \Delta W_i^{\text{on}}, \Delta W_i^{\text{off}} \rangle_F = \operatorname{Tr}\left((\Delta W_i^{\text{on}})^\top \Delta W_i^{\text{off}}\right).
\end{equation}

Taking the transpose:
\begin{equation}
(\Delta W_i^{\text{on}})^\top = (U_i \Delta \Sigma_i Z^\top)^\top = Z \Delta \Sigma_i U_i^\top.
\end{equation}

Substituting into the inner product:
\begin{align}
\langle \Delta W_i^{\text{on}}, \Delta W_i^{\text{off}} \rangle_F 
&= \operatorname{Tr}(Z \Delta \Sigma_i U_i^\top \cdot U_i S_i Z^\top) \\
&= \operatorname{Tr}(Z \Delta \Sigma_i (U_i^\top U_i) S_i Z^\top).
\end{align}

Assuming that $U_i^\top U_i = I$ (which holds exactly if $U_i$ is orthogonal, or approximately under near-orthogonality), we obtain:
\begin{equation}
\langle \Delta W_i^{\text{on}}, \Delta W_i^{\text{off}} \rangle_F = \operatorname{Tr}(Z \Delta \Sigma_i S_i Z^\top).
\end{equation}

Using the cyclic property of the trace:
\begin{equation}
\operatorname{Tr}(Z \Delta \Sigma_i S_i Z^\top) = \operatorname{Tr}(Z^\top Z \Delta \Sigma_i S_i).
\end{equation}

Under the assumption that$ Z^\top Z = I$, this simplifies to:
\begin{equation}
\operatorname{Tr}(\Delta \Sigma_i S_i).
\end{equation}

Since $\Delta \Sigma_i$ is diagonal and $S_i$ is strictly off-diagonal (i.e., $(S_i)_{jj} = 0$ for all $j$), we have:
\begin{equation}
\operatorname{Tr}(\Delta \Sigma_i S_i) = \sum_j (\Delta \Sigma_i)_{jj} (S_i)_{jj} = 0.
\end{equation}

Therefore,
\begin{equation}
\langle \Delta W_i^{\text{on}}, \Delta W_i^{\text{off}} \rangle_F = 0.
\end{equation}

The on-axis and off-axis updates are orthogonal under the Frobenius inner product, which allows us to interpret their effects as geometrically independent. This orthogonality also enables a clean angular decomposition of the total update, as shown in Corollary~\ref{cor:angular_representation}.

\subsubsection*{C.2.3. Proof of Corollary~\ref{cor:angular_representation}: Angular Representation of the Total Update}

We aim to represent the total update $\Delta W_i$ as a combination of two orthogonal unit matrices with a scalar rotation parameter $\alpha$, under a shared projection basis $R_{\text{global}}$.

We start from:
\begin{align}
    \Delta W_i &= (\Delta W_i^{\text{on}} + \Delta W_i^{\text{off}})R_{\text{global}} \\
    &= (U_i \Delta\Sigma_i Z^\top + U_i S_i Z^\top)R_{\text{global}}\\ 
    &= U_i (\Delta\Sigma_i + S_i) Z^\top R_{\text{global}}.
\end{align}

We define the unit direction matrices:
\begin{equation}
\hat{U}_{\text{on}} = \frac{\Delta W_i^{\text{on}}}{\|\Delta W_i^{\text{on}}\|_F}, \quad
\hat{U}_{\text{off}} = \frac{\Delta W_i^{\text{off}}}{\|\Delta W_i^{\text{off}}\|_F},
\end{equation}
and define the angular parameter:
\begin{equation}
\tan \alpha = \frac{\|\Delta W_i^{\text{off}}\|_F}{\|\Delta W_i^{\text{on}}\|_F}.
\end{equation}

Then, the exact angular decomposition becomes:
\begin{equation}
\Delta W_i = \|\Delta W_i\|_F \left( \cos \alpha \cdot \hat{U}_{\text{on}} + \sin \alpha \cdot \hat{U}_{\text{off}} \right).
\end{equation}

\noindent\textbf{Small-angle approximation:} Assuming $\alpha \ll 1$, we have:
\begin{equation}
\cos \alpha \approx 1, \quad \sin \alpha \approx \alpha.
\end{equation}

Furthermore, since $\|\Delta W_i^{\text{off}}\|_F = \|\Delta W_i^{\text{on}}\|_F \cdot \alpha$, the total Frobenius norm becomes:
\begin{equation}
\|\Delta W_i^{\text{on}}\|_F^2 + \alpha^2 \|\Delta W_i^{\text{on}}\|_F^2 = \|\Delta W_i^{\text{on}}\|_F^2 (1 + \alpha^2),
\end{equation}
\begin{equation}
\|\Delta W_i\|_F = \|\Delta W_i^{\text{on}}\|_F \cdot \sqrt{1 + \alpha^2} \approx \|\Delta W_i^{\text{on}}\|_F (1 + \frac{1}{2}\alpha^2).
\end{equation}

Substitute back into the expression:
\begin{align}
\Delta W_i
&\approx \|\Delta W_i^{\text{on}}\|_F (1 + \tfrac{1}{2} \alpha^2) \left( \hat{U}_{\text{on}} + \alpha \cdot \hat{U}_{\text{off}} \right) R_{\text{global}} \\
&\approx \|\Delta W_i^{\text{on}}\|_F \left( \hat{U}_{\text{on}} + \alpha \cdot \hat{U}_{\text{off}} + \tfrac{1}{2}\alpha^2 \cdot \hat{U}_{\text{on}} \right) R_{\text{global}}.
\end{align}

Under small-angle approximation ($\alpha \ll 1$), the total update admits the compact directional form:
\begin{equation}
\Delta W_i \approx \|\Delta W_i^{\text{on}}\|_F \left( \hat{U}_{\text{on}} + \alpha \cdot \hat{U}_{\text{off}} \right) R_{\text{global}},
\end{equation}
which reflects a dominant axis-wise update with a second-order correction along the orthogonal direction. This completes the proof of Corollary~\ref{cor:angular_representation}.

\subsection*{C.3. Brief Summary}
In this section, we establish the theoretical foundations for the decomposition and update structure underlying FRoD. We begin by proving that, due to the general non-commutativity of Gram matrices, a global joint factorization with a shared orthogonal matrix across arbitrary sets of weight matrices is, in general, infeasible. Motivated by this structural limitation, we adopt a more relaxed framework—Hierarchical Joint Decomposition—which leverages QR factorization and eigendecomposition to derive approximately orthogonal left bases and a globally shared latent space. This decomposition serves as the backbone of our subsequent fine-tuning formulation.

On top of this decomposition framework, we propose a two-component update mechanism that separates parameter updates into (i) on-axis components, which maintain the singular value magnitudes, and (ii) off-axis components, which introduce rotational degrees of freedom. We theoretically validate this formulation through three key results: 1) a spectral stability theorem demonstrating that sparse off-axis perturbations preserve the dominant singular spectrum within a bounded deviation; 2) a Frobenius-orthogonality result showing that on-axis and off-axis updates are orthogonal under the Frobenius inner product, thereby enabling a clean geometric decomposition; and 3) an angular representation corollary, which captures the update direction as a low-dimensional rotation, succinctly encoding both magnitude and directional adjustments.

Collectively, these results show that FRoD’s update structure preserves the spectral integrity of the pretrained model while introducing structured and interpretable degrees of freedom. This provides both a theoretical rationale and a principled framework for structurally enhanced parameter-efficient fine-tuning.

\section*{D. Experiment Details}

\subsection*{D.1. Datasets Details}
For reproducibility, all datasets employed in our experiments are sourced from the \textbf{Hugging Face} open-source dataset repository.

\noindent \textbf{Image Classification(IC) tasks}:
\begin{itemize}
    \item \textbf{Cars} \cite{krause_cars_2013} from "tanganke/stanford\_cars" datasets containing 8.14k train images and 8.04k test images across 196 categories.
    \item \textbf{DTD} \cite{cimpoi_dtd_2014} from "tanganke/dtd" datasets containing 3.76k train images and 1.88k test images across 47 categories. 
    \item \textbf{EuroSAT} \cite{helber_eurosat_2019} from "tanganke/eurosat" datasets containing 21.6k train images and 2.7k test images across 10 categories.
    \item \textbf{GTSRB} \cite{stallkamp_gstrb_2011} from "tanganke/gtsrb" datasets containing 26.6k train images and 12.6k test images across 43 categories.
    \item \textbf{RESISC45} \cite{cheng_resisc45_2017} from "tanganke/resisc45" datasets containing 18.9k train images and 1.88k test images across 45 categories.
    \item \textbf{SUN397} \cite{xiao_sun_2016} from "tanganke/sun397" datasets containing 19.9k train images and 19.9k test images across 287 categories.
    \item \textbf{SVHN} \cite{netzer_svnh_nodate} from "ufldl-stanford/svhn" datasets with "cropped\_digits" subset containing 73.3k train images and 26k test images across 10 categories.
\end{itemize}

\noindent \textbf{Commonsense Reasoning(CR) Tasks}:
\begin{itemize}
    \item \textbf{BoolQ} \cite{clark_boolq_2019} from "google/boolq" datasets containing 9.43k train rows and 3.27k validation rows across 2 labels.
    \item \textbf{PIQA} \cite{bisk_piqa_2020} from "baber/piqa" datasets containing 16.1k train rows and 1.84k validation rows across 2 labels. 
    \item \textbf{SIQA} \cite{sap_siqa_2019} from "lighteval/siqa" datasets containing 33.4k train rows and 1.95k validation rows across 3 labels.
    \item \textbf{HellaSwag} \cite{zellers_hellaswag_2019} from "Rowan/hellaswag" datasets containing 39.9k train rows and 10k validation rows across 4 labels.
    \item \textbf{WinoGrande} \cite{sakaguchi_winogrande_2020} from "allenai/winogrande" datasets with subset "winogrande\_l" containing 10.2k train rows and 1.27k validation rows across 2 labels.
    \item \textbf{OBQA} \cite{mihaylov_obqa_2018} from "allenai/openbookqa" datasets with subset "main" containing 4.96k train rows and 500 test rows across 4 labels.
\end{itemize}

Since we evaluate our PEFT method on natural language generation (NLG) models using \textbf{Commonsense Reasoning} tasks, we present our zero-shot prompt here:

\begin{itemize}
    \item \textbf{BoolQ} "Reference document:\\{example['passage']}\\Question:\\{example['question']}"
    \item \textbf{PIQA} "Physical situation description:\\{example['goal']}\\Choice A:\\{example['sol1']}\\Choice B:\\{example['sol2']}\\Question:\\Which solution is more physically plausible, Choice A or Choice B?"
    \item \textbf{SIQA} "Please choose the correct Choice to the question: {example['context']}, Question:\\{example['question']}\\Choice A:\\{example['answerA']}\\Choice B:\\{example['answerB']}\\Choice C:\\{example['answerC']}\\Which choice best answers the question, Choice A, B, or C?"
    \item \textbf{HellaSwag} "Context:\\{example['ctx']}\\{endings}Question:\\Which ending is the most plausible continuation, Ending A, B, C, or D?"
    \item \textbf{WinoGrande} "Please choose the correct answer to fill in the blank to complete the given sentence: {example['sentence']}\\\\Option1: {example['option1']}\\Option2: {example['option2']}."
    \item \textbf{OBQA} "Please choose the correct answer to the question:{example['question\_stem']}\\{opts}."
\end{itemize}

\noindent \textbf{Natural Language Understanding (NLU) Tasks}: We utilize the "nyu-mll/glue" dataset collection, covering seven different tasks: \textbf{CoLA}, \textbf{SST-2}, \textbf{MRPC}, \textbf{QQP}, \textbf{MNLI}, \textbf{QNLI}, and \textbf{RTE}.

\subsection*{D.2. Baseline details}

We adopted a subset of GOAT’s baseline results and obtained the necessary data permissions from the original authors. In addition, we trained the following baseline models:

\begin{itemize}
    \item \textbf{Full FT} refers to fine-tuning the model with all parameters. To better highlight the fast-convergence property of our approach, we trained the Full FT baseline for the same number of epochs and swept several learning rates. We then chose task-specific rates that produced the best average results: 1e-4 for IC, 5e-5 for CR, and 1e-5 for NLU.
    \item \textbf{VeRA} \cite{kopiczko_vera_2023} introduces two frozen and shared random-basis, and finetunes two vectors that significantly reduces the number of trainable parameters compared to LoRA. We train each model for more than ten epochs to ensure that the results on these functions have converged.
    \item \textbf{RandLoRA} \cite{albert_randlora_2025} achieves parameter efficiency and low memory cost while enabling \textbf{full rank} model updates. Similar to VeRA, we use same epochs.
\end{itemize}

\subsection* {D.3. Implementation Details}

All experiments were conducted on a heterogeneous GPU cluster consisting of four NVIDIA RTX 6000 Ada (49 GB) and four NVIDIA V100 (32 GB) units. All LLaMA training runs were performed using bf16 precision.

\begin{table}[t]
\centering
\caption{Hyperparameters of the image classification tasks for FRoD.}
\renewcommand{\arraystretch}{1.2}
\begin{tabular}{l|ccccccc}
\toprule
\textbf{Hyperparameter} & Cars & DTD & EuroSAT & GTSRB & RESISC45 & SUN397 & SVHN \\
\midrule
Batch Size & \multicolumn{7}{c}{128} \\
Optimizer & \multicolumn{7}{c}{AdamW} \\
% Warmup Steps & \multicolumn{7}{c}{0.1} \\
% LR Schedule & \multicolumn{7}{c}{Cosine} \\
Learning Rate($\Sigma_i$) & \multicolumn{7}{c}{5e-3} \\
Learning Rate($S$) & \multicolumn{7}{c}{5e-4} \\
Epochs &   \multicolumn{7}{c}{4} \\
\bottomrule
\end{tabular}
\end{table}

\begin{table}[h]
\centering
\caption{Hyperparameters of the commonsense reasoning tasks for FRoD.}
\renewcommand{\arraystretch}{1.2}
\begin{tabular}{l|cccccc}
\toprule
\textbf{Hyperparameter} & BoolQ & PIQA & SIQA & HellaSwag & WinoGrande & OBQA\\
\midrule
Batch Size & \multicolumn{6}{c}{32} \\
Optimizer & \multicolumn{6}{c}{AdamW} \\
Warmup Steps & \multicolumn{6}{c}{0.1} \\
LR Schedule & \multicolumn{6}{c}{Cosine} \\
Learning Rate($\Sigma_i$) & \multicolumn{6}{c}{1e-4} \\
Learning Rate($S$) & \multicolumn{6}{c}{1e-5} \\
Epochs &   \multicolumn{6}{c}{3} \\
\bottomrule
\end{tabular}
\end{table}

\subsection* {D.4. Hyperparameters setting of FRoD}

We fine-tune our model on each task using a single set of default hyperparameters, without any prior grid search or ablation. Our goal in the main experiments is not to over-optimize for each benchmark, but rather to verify whether our method can perform competitively under minimal tuning. All learning rates, batch sizes, and training epochs were heuristically selected at the outset and remained fixed across runs. Consequently, some main results may underperform those from later ablation studies. However, this minimal configuration setting is a deliberate design choice: it highlights our method’s rapid convergence and robustness under lightweight training.

Tables 1–3 summarize the hyperparameter configurations used across image classification, commonsense reasoning, and natural language understanding (NLU) tasks, respectively. Notably, we employ only 4 epochs for all image classification datasets, 3 epochs for most commonsense and NLU tasks, and 10 epochs only for RTE, which is known to benefit from longer training. These short schedules demonstrate one of our core strengths: our method achieves strong results with dramatically fewer training iterations compared to conventional approaches. For all experiments, we use AdamW as the optimizer, cosine learning rate schedules when applicable.

\begin{table*}[t]
\centering
\caption{Hyperparameters of the natural language understanding tasks for FRoD.}
\renewcommand{\arraystretch}{1.2}
\begin{tabular}{l|ccccccc}
\toprule
\textbf{Hyperparameter} & CoLA & SST-2 & MRPC & QQP & MNLI & QNLI & RTE \\
\midrule
Batch Size & \multicolumn{7}{c}{64}\\
Optimizer & \multicolumn{7}{c}{AdamW} \\
Warmup Steps & \multicolumn{7}{c}{0.1} \\
LR Schedule & \multicolumn{7}{c}{Cosine} \\
Learning Rate($\Sigma_i$) & \multicolumn{7}{c}{5e-4} \\
Learning Rate($S$) & \multicolumn{7}{c}{5e-5} \\
\midrule
Epochs & 3 & 3 & 3 & 3 & 3 & 3 & 10 \\
\bottomrule
\end{tabular}
\end{table*}

\subsection*{D.5. Experiments on Additional Initialization Methods}
Figure~\ref{fig:error} visualizes the per-layer singular value spectra for PiSSA initialization versus our initialization. PiSSA’s layer-wise truncated SVD often omits important large singular directions, leading to uneven signal capture and slower convergence. In contrast, FRoD jointly decomposes across all layers, ensuring that each layer retains its dominant singular modes. As a result, the minimum leading singular value across layers is significantly higher under Joint Decomposition, demonstrating its superior ability to preserve parameter effectiveness in all subspaces.
\begin{figure}[t]
\centering
\includegraphics[width=\linewidth]{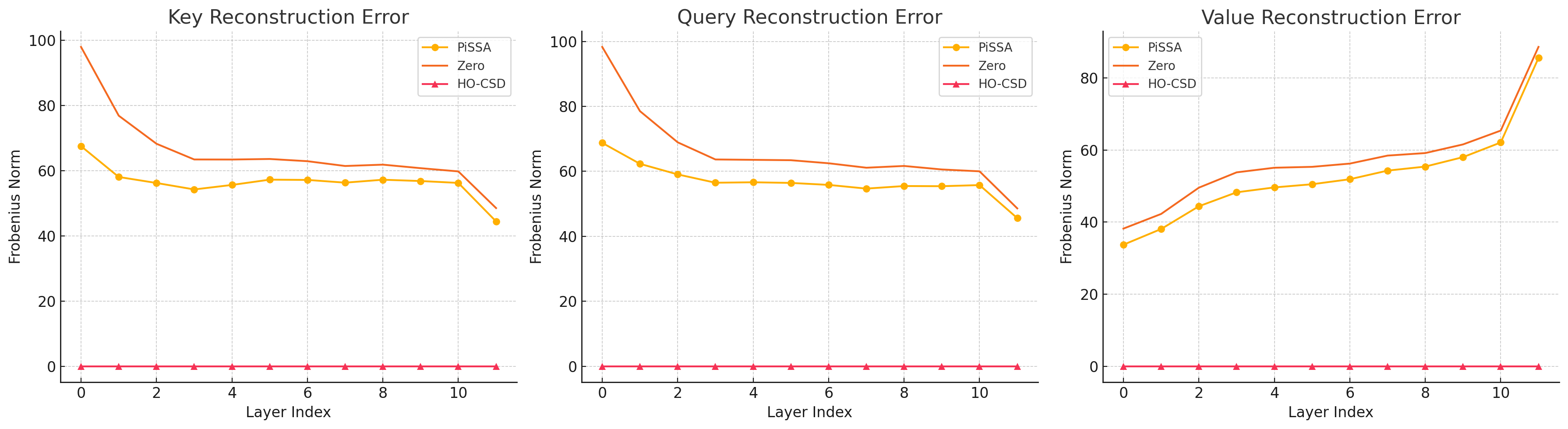}
\caption{Illustration of Reconstruction Error. We use the pretrained CLIP ViT-B/32 model as the initialization matrix to analyze the reconstruction error induced by large singular vector estimation and by our proposed method.}
\label{fig:error}
\end{figure}

\begin{figure}[t]
\centering
\includegraphics[width=\linewidth]{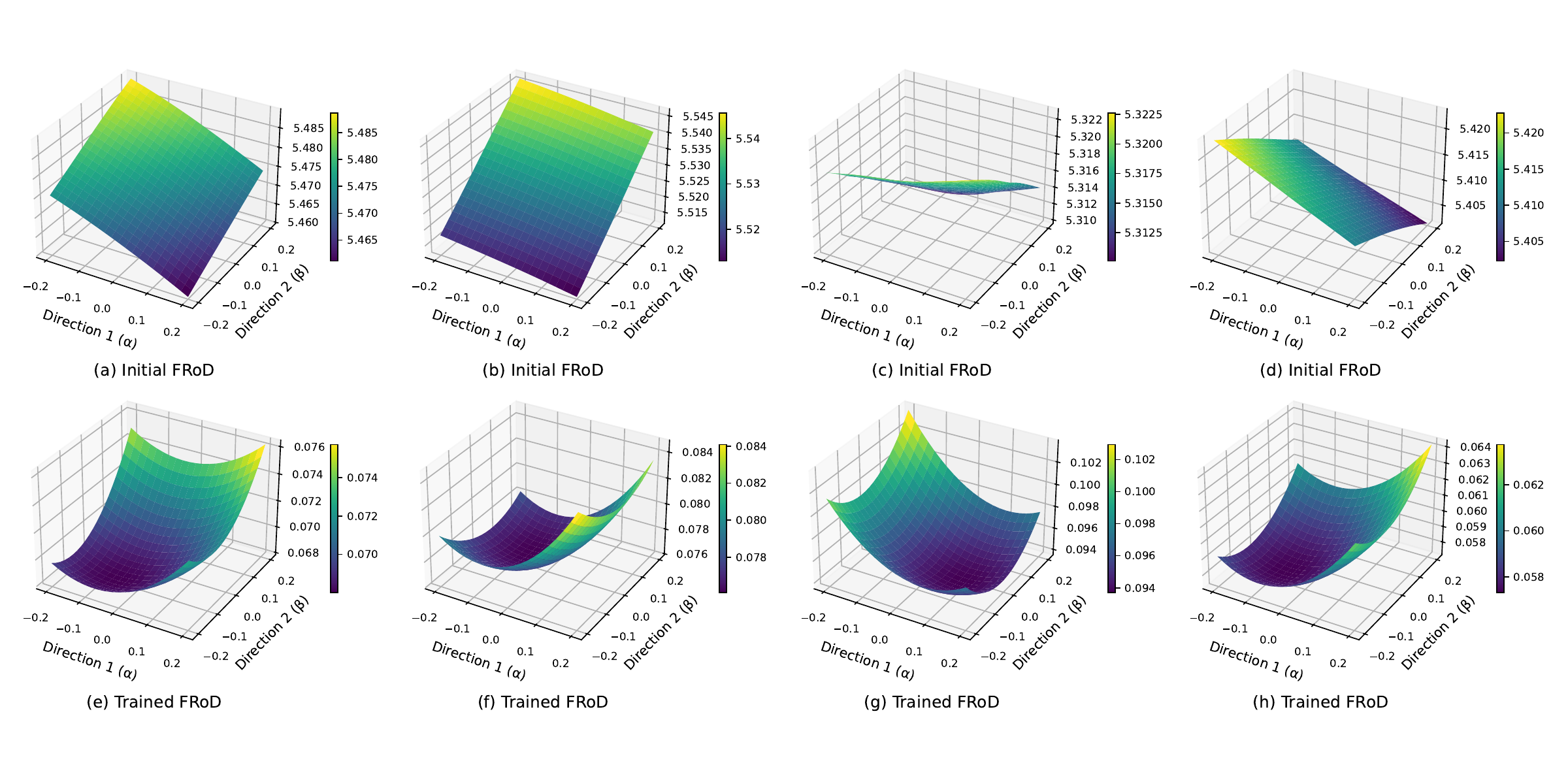}
\caption{Comparison of loss landscapes for 4 different random seeds with sparsity of $0.01$.}
\label{fig:0.01_landscapes}
\vspace{-1.0em}
\end{figure}

\begin{figure}[h]
\centering
\includegraphics[width=\linewidth]{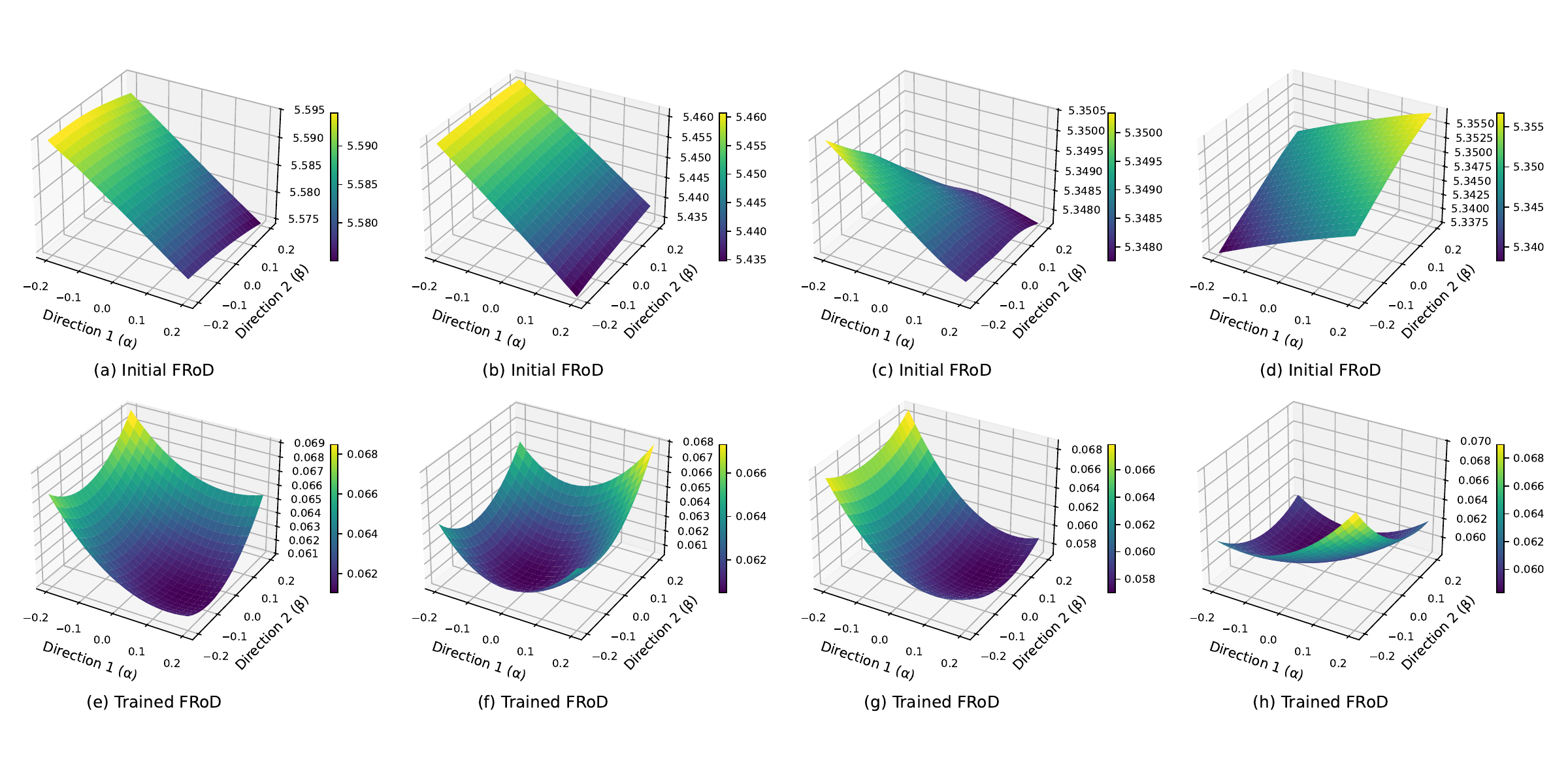}
\caption{Comparison of loss landscapes for 4 different random seeds with sparsity of $0.1$.}
\label{fig:0.1_landscapes}
\vspace{-1.0em}
\end{figure}

\subsection* {D.6. Sparsity for Loss Landscape}
\figurename~\ref{fig:0.01_landscapes} and \ref{fig:0.1_landscapes} illustrate the loss landscapes of the FRoD method under two sparsity settings, $s = 0.01$ and $s = 0.1$, across different random seeds. From the visualizations, we observe a clear trend toward a more rugged optimization surface as sparsity increases. Specifically, when $s = 0.1$, the post-training loss landscape (top of \figurename~\ref{fig:0.1_landscapes}) exhibits steeper curvature and more complex perturbation patterns compared to the smoother, quasi-axially linear structure observed when $s = 0.01$ (\figurename~\ref{fig:0.01_landscapes}). This change indicates that a higher sparsity design expands the expressiveness of the update subspace, allowing it to approximate the expressive power of full-parameter fine-tuning. However, such flexibility also introduces stronger non-convexity (see \figurename~\ref{fig:0.1_landscapes}c and \figurename~\ref{fig:landscape}a), which in turn impacts optimization stability and convergence speed.

This phenomenon is further corroborated by the experimental results in Table 4. The table summarizes average accuracy across epochs under different hyperparameter settings. Notably, when on-axis strength updates are disabled (i.e., $\mathrm{lr}(\Sigma_i) = 0$), the best average accuracy for $s = 0.1$ is significantly lower than that for $s = 0.02$. For example, at epoch 10, the best configuration for $s = 0.1$ reaches only 46.95\%, whereas $s = 0.02$ achieves 49.62\%; similarly, in another task at epoch 10, $s = 0.1$ achieves 72.77\% compared to 73.12\% for $s = 0.02$, with the gap gradually narrowing due to the influence of the larger parameter space. These results suggest that although higher sparsity theoretically provides greater parameter freedom, the resulting ruggedness of the optimization landscape can hinder convergence speed and final performance.

In summary, excessive sparsity may lead to overexpansion of the update space, introducing unstable directions and impeding the optimization process. Therefore, a moderate sparsity level (e.g., $s = 0.01$ or $0.02$) is more conducive to achieving a balance between expressiveness and optimization stability.

\section*{E. Ablation Study Details}

\begin{table*}[t]
\centering
\scriptsize
\renewcommand{\arraystretch}{1.1}
\begin{tabular}{c c c|c c c || c c c|c c c}
\toprule
\multicolumn{3}{c|}{\textbf{Hyperparameters}} & \multicolumn{3}{c||}{\textbf{Avg.\ Acc. (epoch) (\%)}} & 
\multicolumn{3}{c|}{\textbf{Hyperparameters}} & \multicolumn{3}{c}{\textbf{Avg.\ Acc. (epoch) (\%)}} \\
\cmidrule(lr){1-3} \cmidrule(lr){4-6} \cmidrule(lr){7-9} \cmidrule(lr){10-12}
$s$ & $\mathrm{lr}(S)$ & $\mathrm{lr}(\Sigma_i)$ & 1 & 4 & 10 &
$s$ & $\mathrm{lr}(S)$ & $\mathrm{lr}(\Sigma_i)$ & 1 & 4 & 10 \\
\midrule

$0.0$ & 0 & 1e-04 & 25.89 & 55.89 & 70.77 & $0.01$ & 1e-05 & 0 & 29.83 & 62.24 & 71.17\\
$0.0$ & 0 & 5e-04 & 34.20 & 66.52 & 71.26 & $0.01$ & 5e-05 & 0 & 44.38 & 68.71 & 71.49\\
$0.0$ & 0 & 1e-03 & 42.13 & 68.58 & 70.86 & $0.01$ & 1e-04 & 0 & 47.72 & 71.33 & 72.88\\
$0.0$ & 0 & 5e-03 & 50.03 & 65.66 & 67.28 & $0.01$ & 5e-04 & 0 & 42.29 & 53.69 & 59.18\\
 - & - & - & - & - & - & $0.02$ & 1e-05 & 0 & 29.96 & 62.47 & 70.77\\
 - & - & - & - & - & - & $0.02$ & 5e-05 & 0 & 44.51 & 68.87 & 70.92\\
 - & - & - & - & - & - & $0.02$ & 1e-04 & 0 & 49.62 & 70.89 & 73.12\\
 - & - & - & - & - & - & $0.02$ & 5e-04 & 0 & 41.59 & 55.28 & 60.60\\
 - & - & - & - & - & - & $0.1$ & 1e-05 & 0 & 29.89 & 61.65 & 70.97\\
 - & - & - & - & - & - & $0.1$ & 5e-05 & 0 & 43.62 & 68.54 & 71.35\\
 - & - & - & - & - & - & $0.1$ & 1e-04 & 0 & 46.95 & 70.68 & 72.77\\
 - & - & - & - & - & - & $0.1$ & 5e-04 & 0 & 40.64 & 56.43 & 61.20\\
 \midrule
$0.1$ & 5e-04 & 5e-03 & 37.47 & 48.51 & 52.22 & $0.1$ & 5e-04 & 1e-03 & 43.82 & 57.32 & 61.83\\
$0.1$ & 5e-04 & 5e-04 & 42.74 & 54.27 & 59.95 & $0.1$ & 5e-04 & 1e-04 & 41.73 & 54.42 & 60.30\\
$0.1$ & 1e-04 & 5e-03 & 50.40 & 68.23 & 70.12 & $0.1$ & 1e-04 & 1e-03 & 51.78 & 72.36 & 74.74\\
$0.1$ & 1e-04 & 5e-04 & 49.31 & 71.17 & 73.91 & $0.1$ & 1e-04 & 1e-04 & 49.01 & 70.65 & 72.78\\
$0.1$ & 5e-05 & 5e-03 & 49.86 & 67.22 & 69.76 & $0.1$ & 5e-05 & 1e-03 & 49.57 & 72.37 & 74.09\\
$0.1$ & 5e-05 & 5e-04 & 45.31 & 71.25 & 73.01 & $0.1$ & 5e-05 & 1e-04 & 45.03 & 70.02 & 71.41\\
$0.1$ & 1e-05 & 5e-03 & 50.56 & 65.96 & 68.25 & $0.1$ & 1e-05 & 1e-03 & 42.80 & 69.47 & 71.74\\
$0.1$ & 1e-05 & 5e-04 & 39.83 & 68.12 & 71.26 & $0.1$ & 1e-05 & 1e-04 & 32.75 & 63.98 & 71.00\\
$0.02$ & 5e-04 & 5e-03 & 38.99 & 58.28 & 61.53 & $0.02$ & 5e-04 & 1e-03 & 41.80 & 56.82 & 61.80\\
$0.02$ & 5e-04 & 5e-04 & 40.73 & 55.17 & 61.26 & $0.02$ & 5e-04 & 1e-04 & 39.24 & 54.96 & 60.93\\
$0.02$ & 1e-04 & 5e-03 & 51.87 & 67.55 & 70.55 & $0.02$ & 1e-04 & 1e-03 & 48.36 & 71.85 & 74.59\\
$0.02$ & 1e-04 & 5e-04 & 51.39 & 72.32 & 74.47 & $0.02$ & 1e-04 & 1e-04 & 50.19 & 71.71 & 73.89\\
$0.02$ & 5e-05 & 5e-03 & 51.76 & 67.42 & 70.52 & $0.02$ & 5e-05 & 1e-03 & 48.42 & 71.81 & 73.76\\
$0.02$ & 5e-05 & 5e-04 & 46.70 & 70.10 & 72.79 & $0.02$ & 5e-05 & 1e-04 & 45.10 & 69.84 & 71.94\\
$0.02$ & 1e-05 & 5e-03 & 49.62 & 66.81 & 68.32 & $0.02$ & 1e-05 & 1e-03 & 43.15 & 69.11 & 70.90\\
$0.02$ & 1e-05 & 5e-04 & 39.11 & 68.22 & 71.31 & $0.02$ & 1e-05 & 1e-04 & 33.06 & 64.33 & 71.05\\
$0.01$ & 5e-04 & 5e-03 & 38.85 & 49.07 & 51.65 & $0.01$ & 5e-04 & 1e-03 & 41.31 & 52.84 & 59.14\\
$0.01$ & 5e-04 & 5e-04 & 39.68 & 51.43 & 58.09 & $0.01$ & 5e-04 & 1e-04 & 40.30 & 57.91 & 61.92\\
$0.01$ & 1e-04 & 5e-03 & 48.56 & 65.83 & 69.07 & $0.01$ & 1e-04 & 1e-03 & 52.25 & 71.72 & 73.93\\
$0.01$ & 1e-04 & 5e-04 & 49.39 & 71.35 & 73.91 & $0.01$ & 1e-04 & 1e-04 & 48.56 & 70.65 & 72.05\\
$0.01$ & 5e-05 & 5e-03 & 51.45 & 68.13 & 70.91 & $0.01$ & 5e-05 & 1e-03 & 48.78 & 72.17 & 74.16\\
$0.01$ & 5e-05 & 5e-04 & 47.07 & 71.23 & 73.49 & $0.01$ & 5e-05 & 1e-04 & 42.92 & 69.76 & 71.79\\
$0.01$ & 1e-05 & 5e-03 & 49.70 & 65.32 & 67.58 & $0.01$ & 1e-05 & 1e-03 & 41.86 & 69.50 & 72.02\\
$0.01$ & 1e-05 & 5e-04 & 38.62 & 67.73 & 71.25 & $0.01$ & 1e-05 & 1e-04 & 32.36 & 64.16 & 70.64\\
\bottomrule
\end{tabular}
\caption{Full ablation results (averaged) with convergence markers.}
\label{tab:ablation_avg}
\end{table*}

In the following, we explain on the reason behind our hyperparameter selection. 

\begin{itemize}
    \item \textbf{Epoch} (1, 4, 10) correspond to early-stage training, the truncation setting employed in our main results, and the minimal truncation configuration adopted by existing SOTA baselines, respectively.
    \item \textbf{Sparsity s} (0.01, 0.02, 0.1) correspond to LoRA settings with rank = 8 (small), 16 (base), and 64 (large), respectively, under a $768 \times 768$ parameter space of the CLIP ViT-B/32 model. Notably, $s = 0.1$ represents a sparsity level an order of magnitude higher than $s = 0.01$, which enables a meaningful comparison in our analysis of sparsity-induced performance differences.
    \item \textbf{Learning Rate lr} (off-axis learning rate $lr(S) \in \{\text{1e-5}, \text{5e-5}, \text{1e-4}, \text{5e-4}\}$, and on-axis learning rate $lr(\Sigma_i) \in \{ \text{1e-4}, \text{5e-4}, \text{1e-3}, \text{5e-3}\}$). Based on the preceding analysis, we find that maintaining a higher on-axis learning rate than the off-axis counterpart helps preserve the on-axis strength, thereby sustaining the anisotropic linearity of the loss landscape. To this end, we deliberately configure $lr(\Sigma_i)$ to be one order of magnitude greater than $lr(S)$ across all settings. Furthermore, we design a set of intersecting configurations to empirically verify the theoretical implications of this learning rate hierarchy.
\end{itemize}

\begin{table*}[!t]
\centering
\scriptsize
\renewcommand{\arraystretch}{1.1}
\begin{tabular}{c c c|c c c || c c c|c c c}
\toprule
\multicolumn{3}{c|}{\textbf{Hyperparameters}} & \multicolumn{3}{c||}{\textbf{Acc. (epoch) (\%)}} & 
\multicolumn{3}{c|}{\textbf{Hyperparameters}} & \multicolumn{3}{c}{\textbf{Acc. (epoch) (\%)}} \\
\cmidrule(lr){1-3} \cmidrule(lr){4-6} \cmidrule(lr){7-9} \cmidrule(lr){10-12}
$s$ & $\mathrm{lr}(S)$ & $\mathrm{lr}(\Sigma_i)$ & 1 & 4 & 10 &
$s$ & $\mathrm{lr}(S)$ & $\mathrm{lr}(\Sigma_i)$ & 1 & 4 & 10 \\
\midrule

$0.0$ & 0 & 1e-04 & 12.29 & 56.01* & 70.32* & $0.01$ & 5e-05 & 0 & 51.28 & 70.80* & 74.95\\
$0.0$ & 0 & 5e-04 & 30.05 & 67.13* & 72.87 & $0.01$ & 5e-04 & 0 & 56.38 & 65.69* & 68.56\\
$0.0$ & 0 & 1e-03 & 45.80 & 71.60* & 73.51 & $0.01$ & 1e-04 & 0 & 57.18 & 74.36* & 76.28*\\
$0.0$ & 0 & 5e-03 & 60.05 & 70.48 & 71.38 & $0.01$ & 1e-05 & 0 & 25.05 & 64.89* & 72.71*\\
 - & - & - & - & - & - & $0.02$ & 1e-04 & 0 & 57.87 & 73.46* & 75.85*\\
 - & - & - & - & - & - & $0.02$ & 5e-04 & 0 & 56.01 & 67.18* & 66.91\\
 % &  &  &  &  &  & $0.02$ & 1e-04 & 0 & 59.95 & 73.03* & 76.86\\
 % &  &  &  &  &  & $0.02$ & 5e-04 & 0 & 56.28 & 68.40* & 69.95\\
 - & - & - & - & - & - & $0.02$ & 5e-05 & 0 & 49.84 & 71.91* & 74.04\\
 - & - & - & - & - & - & $0.02$ & 1e-05 & 0 & 23.09 & 63.99* & 71.86\\
 - & - & - & - & - & - & $0.1$ & 1e-04 & 0 & 55.80 & 72.82* & 75.53\\
 - & - & - & - & - & - & $0.1$ & 5e-04 & 0 & 60.11 & 67.55* & 69.73\\
 % &  &  &  &  &  & $0.1$ & 1e-04 & 0 & 56.54 & 71.12* & 75.53\\
 % &  &  &  &  &  & $0.1$ & 5e-04 & 0 & 59.20 & 64.68 & 67.82\\
 - & - & - & - & - & - & $0.1$ & 5e-05 & 0 & 45.21 & 70.96* & 74.57\\
 - & - & - & - & - & - & $0.1$ & 1e-05 & 0 & 25.32 & 61.44* & 71.12*\\
 \midrule
$0.1$ & 5e-04 & 5e-03 & 54.84 & 63.88* & 65.43 & $0.1$ & 5e-04 & 1e-03 & 54.95 & 66.91* & 68.35\\
$0.1$ & 5e-04 & 5e-04 & 53.30 & 64.47 & 66.60 & $0.1$ & 5e-04 & 1e-04 & 53.94 & 65.64* & 68.09\\
$0.1$ & 1e-04 & 5e-03 & 60.43 & 72.02* & 72.34 & $0.1$ & 1e-04 & 1e-03 & 61.01 & 72.50* & 76.01*\\
$0.1$ & 1e-04 & 5e-04 & 56.65 & 72.61* & 76.60* & $0.1$ & 1e-04 & 1e-04 & 58.62 & 72.82* & 76.12\\
$0.1$ & 5e-05 & 5e-03 & 58.72 & 70.96* & 71.22 & $0.1$ & 5e-05 & 1e-03 & 56.86 & 72.98* & 74.89*\\
$0.1$ & 5e-05 & 5e-04 & 53.24 & 72.29* & 74.79* & $0.1$ & 5e-05 & 1e-04 & 50.85 & 72.29* & 73.30\\
$0.1$ & 1e-05 & 5e-03 & 61.33 & 71.01 & 71.86 & $0.1$ & 1e-05 & 1e-03 & 43.09 & 70.27* & 73.35\\
$0.1$ & 1e-05 & 5e-04 & 40.74 & 70.05* & 73.88* & $0.1$ & 1e-05 & 1e-04 & 30.69 & 65.21* & 71.97*\\
$0.02$ & 5e-04 & 5e-03 & 53.14 & 64.41 & 65.21 & $0.02$ & 5e-04 & 1e-03 & 56.22 & 68.88* & 67.93\\
$0.02$ & 5e-04 & 5e-04 & 56.06 & 66.91* & 69.15 & $0.02$ & 5e-04 & 1e-04 & 53.62 & 67.45 & 67.77*\\
$0.02$ & 1e-04 & 5e-03 & 60.48 & 70.32 & 71.38 & $0.02$ & 1e-04 & 1e-03 & 59.57 & 72.66* & 76.60\\
$0.02$ & 1e-04 & 5e-04 & 58.56 & 72.93 & 76.54 & $0.02$ & 1e-04 & 1e-04 & 59.89 & 72.71 & 74.95\\
$0.02$ & 5e-05 & 5e-03 & 60.53 & 70.74 & 70.90 & $0.02$ & 5e-05 & 1e-03 & 56.28 & 73.40* & 75.69\\
$0.02$ & 5e-05 & 5e-04 & 52.55 & 71.01* & 74.10* & $0.02$ & 5e-05 & 1e-04 & 53.46 & 71.49* & 74.20\\
$0.02$ & 1e-05 & 5e-03 & 60.27 & 69.95* & 70.74 & $0.02$ & 1e-05 & 1e-03 & 47.45 & 71.76* & 72.45\\
$0.02$ & 1e-05 & 5e-04 & 41.33 & 68.09* & 71.86* & $0.02$ & 1e-05 & 1e-04 & 31.06 & 65.59* & 72.39\\
$0.01$ & 5e-04 & 5e-03 & 51.86 & 63.99 & 66.12 & $0.01$ & 5e-04 & 1e-03 & 52.45 & 65.16 & 68.46\\
$0.01$ & 5e-04 & 5e-04 & 57.71 & 67.18 & 69.20 & $0.01$ & 5e-04 & 1e-04 & 56.91 & 66.22* & 68.51\\
$0.01$ & 1e-04 & 5e-03 & 64.20 & 71.17* & 71.70 & $0.01$ & 1e-04 & 1e-03 & 59.52 & 74.47* & 77.18\\
$0.01$ & 1e-04 & 5e-04 & 57.61 & 73.83* & 75.53 & $0.01$ & 1e-04 & 1e-04 & 59.04 & 73.94* & 75.80\\
$0.01$ & 5e-05 & 5e-03 & 60.74 & 71.54* & 72.13 & $0.01$ & 5e-05 & 1e-03 & 55.59 & 72.45* & 74.41\\
$0.01$ & 5e-05 & 5e-04 & 52.39 & 72.29* & 74.95 & $0.01$ & 5e-05 & 1e-04 & 48.19 & 73.14* & 74.52\\
$0.01$ & 1e-05 & 5e-03 & 59.73 & 70.16* & 70.48 & $0.01$ & 1e-05 & 1e-03 & 43.35 & 69.57 & 73.51\\
$0.01$ & 1e-05 & 5e-04 & 39.89 & 69.79* & 73.24 & $0.01$ & 1e-05 & 1e-04 & 28.40 & 64.73* & 72.77*\\
\bottomrule
\end{tabular}
\caption{Ablation results for dtd accuracy with convergence marks.}
\label{tab:ablation_dtd_accuracy}
\end{table*}

\begin{table*}[t]
\centering
\scriptsize
\renewcommand{\arraystretch}{1.1}
\begin{tabular}{c c c|c c c || c c c|c c c}
\toprule
\multicolumn{3}{c|}{\textbf{Hyperparameters}} & \multicolumn{3}{c||}{\textbf{Acc. (epoch) (\%)}} & 
\multicolumn{3}{c|}{\textbf{Hyperparameters}} & \multicolumn{3}{c}{\textbf{Acc. (epoch) (\%)}} \\
\cmidrule(lr){1-3} \cmidrule(lr){4-6} \cmidrule(lr){7-9} \cmidrule(lr){10-12}
$s$ & $\mathrm{lr}(S)$ & $\mathrm{lr}(\Sigma_i)$ & 1 & 4 & 10 &
$s$ & $\mathrm{lr}(S)$ & $\mathrm{lr}(\Sigma_i)$ & 1 & 4 & 10 \\
\midrule

$0.0$ & 0 & 1e-04 & 67.94 & 85.00* & 89.90* & $0.01$ & 1e-04 & 0 & 87.46 & 90.86 & 91.24\\
$0.0$ & 0 & 5e-04 & 78.32 & 89.71* & 90.95 & $0.01$ & 5e-05 & 0 & 83.40 & 91.30* & 92.40\\
$0.0$ & 0 & 1e-03 & 83.78 & 90.94 & 91.78 & $0.01$ & 5e-04 & 0 & 84.16 & 88.48* & 88.60\\
$0.0$ & 0 & 5e-03 & 87.35 & 88.95 & 89.05 & $0.01$ & 1e-05 & 0 & 72.37 & 88.67* & 90.24\\
 - & - & - & - & - & - & $0.02$ & 1e-04 & 0 & 89.08 & 91.19 & 92.73\\
 - & - & - & - & - & - & $0.02$ & 5e-04 & 0 & 86.70 & 88.29 & 88.87\\
 - & - & - & - & - & - & $0.02$ & 5e-05 & 0 & 84.97 & 91.10 & 92.30\\
 - & - & - & - & - & - & $0.02$ & 1e-05 & 0 & 73.67 & 88.79* & 90.63*\\
 - & - & - & - & - & - & $0.1$ & 1e-04 & 0 & 87.83 & 91.25 & 92.46\\
 - & - & - & - & - & - & $0.1$ & 5e-04 & 0 & 83.14 & 87.89 & 88.41\\
 - & - & - & - & - & - & $0.1$ & 1e-05 & 0 & 71.98 & 88.03* & 90.37\\
 - & - & - & - & - & - & $0.1$ & 5e-05 & 0 & 85.57 & 91.60* & 92.54*\\
 \midrule
$0.1$ & 5e-04 & 5e-03 & 82.75 & 87.11 & 88.24 & $0.1$ & 5e-04 & 1e-03 & 86.25 & 89.40* & 89.35\\
$0.1$ & 5e-04 & 5e-04 & 85.60 & 87.38 & 89.22* & $0.1$ & 5e-04 & 1e-04 & 86.33 & 89.25 & 88.54*\\
$0.1$ & 1e-04 & 5e-03 & 87.00 & 90.60* & 90.89* & $0.1$ & 1e-04 & 1e-03 & 88.21 & 92.43 & 92.32\\
$0.1$ & 1e-04 & 5e-04 & 86.27 & 91.81* & 92.02 & $0.1$ & 1e-04 & 1e-04 & 87.79 & 91.75* & 92.56*\\
$0.1$ & 5e-05 & 5e-03 & 88.16 & 89.81 & 91.06 & $0.1$ & 5e-05 & 1e-03 & 88.00 & 91.60* & 92.43\\
$0.1$ & 5e-05 & 5e-04 & 85.90 & 91.97* & 93.46 & $0.1$ & 5e-05 & 1e-04 & 84.62 & 90.89 & 92.70*\\
$0.1$ & 1e-05 & 5e-03 & 85.40 & 89.30 & 89.81 & $0.1$ & 1e-05 & 1e-03 & 84.17 & 90.52* & 92.02\\
$0.1$ & 1e-05 & 5e-04 & 80.17 & 90.16* & 91.22* & $0.1$ & 1e-05 & 1e-04 & 74.78 & 88.97* & 90.70\\
$0.02$ & 5e-04 & 5e-03 & 85.81 & 87.52 & 88.27 & $0.02$ & 5e-04 & 1e-03 & 83.63 & 88.35 & 90.57\\
$0.02$ & 5e-04 & 5e-04 & 87.95 & 88.60 & 89.06 & $0.02$ & 5e-04 & 1e-04 & 84.33 & 88.30* & 89.24*\\
$0.02$ & 1e-04 & 5e-03 & 87.83 & 89.83 & 90.92* & $0.02$ & 1e-04 & 1e-03 & 87.98 & 91.43* & 92.37\\
$0.02$ & 1e-04 & 5e-04 & 88.67 & 92.33* & 92.67 & $0.02$ & 1e-04 & 1e-04 & 87.40 & 91.84* & 93.03*\\
$0.02$ & 5e-05 & 5e-03 & 87.95 & 89.71 & 90.46 & $0.02$ & 5e-05 & 1e-03 & 86.37 & 91.25 & 91.79\\
$0.02$ & 5e-05 & 5e-04 & 85.62 & 91.94* & 93.08* & $0.02$ & 5e-05 & 1e-04 & 85.29 & 91.71* & 92.83\\
$0.02$ & 1e-05 & 5e-03 & 85.97 & 89.11 & 89.41 & $0.02$ & 1e-05 & 1e-03 & 83.13 & 90.63* & 91.92\\
$0.02$ & 1e-05 & 5e-04 & 81.54 & 90.71* & 92.02 & $0.02$ & 1e-05 & 1e-04 & 75.57 & 88.75* & 90.43\\
$0.01$ & 5e-04 & 5e-03 & 87.00 & 87.71 & 88.37 & $0.01$ & 5e-04 & 1e-03 & 86.84 & 88.29* & 89.21\\
$0.01$ & 5e-04 & 5e-04 & 85.46 & 88.08 & 89.73 & $0.01$ & 5e-04 & 1e-04 & 84.24 & 89.52 & 88.75\\
$0.01$ & 1e-04 & 5e-03 & 87.94 & 90.40 & 90.38 & $0.01$ & 1e-04 & 1e-03 & 87.83 & 91.90 & 92.90\\
$0.01$ & 1e-04 & 5e-04 & 87.37 & 91.43 & 91.81* & $0.01$ & 1e-04 & 1e-04 & 87.44 & 91.35 & 91.94\\
$0.01$ & 5e-05 & 5e-03 & 86.11 & 89.98* & 90.86 & $0.01$ & 5e-05 & 1e-03 & 87.79 & 91.87 & 92.68\\
$0.01$ & 5e-05 & 5e-04 & 86.56 & 91.43* & 92.19 & $0.01$ & 5e-05 & 1e-04 & 84.83 & 91.57* & 92.73\\
$0.01$ & 1e-05 & 5e-03 & 86.38 & 89.17 & 89.52 & $0.01$ & 1e-05 & 1e-03 & 84.46 & 91.76* & 91.92\\
$0.01$ & 1e-05 & 5e-04 & 81.86 & 90.35* & 91.24 & $0.01$ & 1e-05 & 1e-04 & 73.81 & 88.59* & 90.44\\
\bottomrule
\end{tabular}
\caption{Ablation results for resisc45 accuracy with convergence marks.}
\label{tab:ablation_resisc45_accuracy}
\end{table*}

\begin{table*}[!t]
\centering
\scriptsize
\renewcommand{\arraystretch}{1.1}
\begin{tabular}{c c c|c c c || c c c|c c c}
\toprule
\multicolumn{3}{c|}{\textbf{Hyperparameters}} & \multicolumn{3}{c||}{\textbf{Acc. (epoch) (\%)}} & 
\multicolumn{3}{c|}{\textbf{Hyperparameters}} & \multicolumn{3}{c}{\textbf{Acc. (epoch) (\%)}} \\
\cmidrule(lr){1-3} \cmidrule(lr){4-6} \cmidrule(lr){7-9} \cmidrule(lr){10-12}
$s$ & $\mathrm{lr}(S)$ & $\mathrm{lr}(\Sigma_i)$ & 1 & 4 & 10 &
$s$ & $\mathrm{lr}(S)$ & $\mathrm{lr}(\Sigma_i)$ & 1 & 4 & 10 \\
\midrule

$0.0$ & 0 & 1e-04 & 4.08 & 26.54* & 59.42* & $0.01$ & 5e-04 & 0 & 0.65 & 11.89* & 33.12\\
$0.0$ & 0 & 5e-04 & 5.09 & 50.75* & 60.69 & $0.01$ & 1e-04 & 0 & 15.38 & 62.74* & 66.41\\
$0.0$ & 0 & 1e-03 & 9.54 & 53.84* & 59.11 & $0.01$ & 1e-05 & 0 & 3.30 & 40.01* & 61.32*\\
$0.0$ & 0 & 5e-03 & 20.31 & 53.84* & 57.56* & $0.01$ & 5e-05 & 0 & 13.47 & 55.96* & 60.65\\
 - & - & - & - & - & - & $0.02$ & 1e-04 & 0 & 19.60 & 61.25* & 64.64\\
 - & - & - & - & - & - & $0.02$ & 5e-04 & 0 & 1.52 & 13.57* & 37.63*\\
 - & - & - & - & - & - & $0.02$ & 1e-05 & 0 & 4.97 & 41.74* & 60.12*\\
 - & - & - & - & - & - & $0.02$ & 5e-05 & 0 & 14.50 & 55.11* & 59.52\\
 - & - & - & - & - & - & $0.1$ & 1e-04 & 0 & 11.90 & 62.31* & 65.94\\
 - & - & - & - & - & - & $0.1$ & 5e-04 & 0 & 1.34 & 10.05* & 31.51*\\
 - & - & - & - & - & - & $0.1$ & 1e-05 & 0 & 5.56 & 41.87* & 61.90*\\
 - & - & - & - & - & - & $0.1$ & 5e-05 & 0 & 14.04 & 55.15* & 60.05\\
\midrule
$0.1$ & 5e-04 & 5e-03 & 1.12 & 12.62* & 25.03* & $0.1$ & 5e-04 & 1e-03 & 1.11 & 23.87* & 41.45*\\
$0.1$ & 5e-04 & 5e-04 & 1.01 & 16.06* & 35.80* & $0.1$ & 5e-04 & 1e-04 & 1.19 & 16.39* & 38.35\\
$0.1$ & 1e-04 & 5e-03 & 26.86 & 62.01* & 66.76* & $0.1$ & 1e-04 & 1e-03 & 16.38 & 66.70* & 71.00\\
$0.1$ & 1e-04 & 5e-04 & 19.08 & 63.49* & 67.42 & $0.1$ & 1e-04 & 1e-04 & 17.85 & 60.81* & 64.21\\
$0.1$ & 5e-05 & 5e-03 & 18.08 & 57.85* & 65.44* & $0.1$ & 5e-05 & 1e-03 & 18.34 & 65.05* & 68.39*\\
$0.1$ & 5e-05 & 5e-04 & 11.89 & 62.94* & 65.50 & $0.1$ & 5e-05 & 1e-04 & 13.19 & 59.08* & 61.40\\
$0.1$ & 1e-05 & 5e-03 & 16.52 & 50.52* & 56.54 & $0.1$ & 1e-05 & 1e-03 & 11.62 & 58.50* & 62.23\\
$0.1$ & 1e-05 & 5e-04 & 8.93 & 54.00* & 60.81 & $0.1$ & 1e-05 & 1e-04 & 5.41 & 44.41* & 60.65*\\
$0.02$ & 5e-04 & 5e-03 & 4.03 & 45.36* & 56.16 & $0.02$ & 5e-04 & 1e-03 & 0.58 & 23.89* & 41.98\\
$0.02$ & 5e-04 & 5e-04 & 1.14 & 21.58* & 40.83* & $0.02$ & 5e-04 & 1e-04 & 1.50 & 21.66* & 42.20*\\
$0.02$ & 1e-04 & 5e-03 & 23.50 & 58.21* & 66.82* & $0.02$ & 1e-04 & 1e-03 & 7.41 & 64.41* & 69.23\\
$0.02$ & 1e-04 & 5e-04 & 20.37 & 66.29* & 69.12 & $0.02$ & 1e-04 & 1e-04 & 16.63 & 63.85* & 68.15\\
$0.02$ & 5e-05 & 5e-03 & 21.43 & 55.83* & 66.31 & $0.02$ & 5e-05 & 1e-03 & 14.13 & 62.31* & 66.17\\
$0.02$ & 5e-05 & 5e-04 & 13.69 & 58.00* & 63.51 & $0.02$ & 5e-05 & 1e-04 & 10.11 & 58.38* & 62.84\\
$0.02$ & 1e-05 & 5e-03 & 16.47 & 56.17* & 59.53 & $0.02$ & 1e-05 & 1e-03 & 9.30 & 55.74* & 60.90\\
$0.02$ & 1e-05 & 5e-04 & 7.62 & 55.55* & 61.76 & $0.02$ & 1e-05 & 1e-04 & 5.21 & 45.89* & 60.32*\\
$0.01$ & 5e-04 & 5e-03 & 0.80 & 5.57* & 12.93* & $0.01$ & 5e-04 & 1e-03 & 0.57 & 11.29* & 32.00*\\
$0.01$ & 5e-04 & 5e-04 & 0.99 & 9.55* & 32.10* & $0.01$ & 5e-04 & 1e-04 & 1.59 & 33.17* & 47.73\\
$0.01$ & 1e-04 & 5e-03 & 2.26 & 45.19* & 54.61 & $0.01$ & 1e-04 & 1e-03 & 21.43 & 62.89* & 66.77\\
$0.01$ & 1e-04 & 5e-04 & 18.28 & 65.18* & 70.30 & $0.01$ & 1e-04 & 1e-04 & 16.14 & 62.60* & 64.83\\
$0.01$ & 5e-05 & 5e-03 & 23.41 & 58.40* & 67.22* & $0.01$ & 5e-05 & 1e-03 & 14.79 & 63.81* & 67.85\\
$0.01$ & 5e-05 & 5e-04 & 15.79 & 61.72* & 66.31 & $0.01$ & 5e-05 & 1e-04 & 11.67 & 57.36* & 62.59\\
$0.01$ & 1e-05 & 5e-03 & 19.13 & 53.12* & 58.48* & $0.01$ & 1e-05 & 1e-03 & 8.46 & 56.59* & 61.42\\
$0.01$ & 1e-05 & 5e-04 & 6.99 & 52.24* & 60.05 & $0.01$ & 1e-05 & 1e-04 & 6.14 & 45.47* & 60.24\\
\bottomrule
\end{tabular}
\caption{Ablation results for cars accuracy with convergence marks.}
\label{tab:ablation_stanford_cars_accuracy}
\end{table*}

\begin{table*}[!t]
\centering
\scriptsize
\renewcommand{\arraystretch}{1.1}
\begin{tabular}{c c c|c c c || c c c|c c c}
\toprule
\multicolumn{3}{c|}{\textbf{Hyperparameters}} & \multicolumn{3}{c||}{\textbf{Acc. (epoch) (\%)}} & 
\multicolumn{3}{c|}{\textbf{Hyperparameters}} & \multicolumn{3}{c}{\textbf{Acc. (epoch) (\%)}} \\
\cmidrule(lr){1-3} \cmidrule(lr){4-6} \cmidrule(lr){7-9} \cmidrule(lr){10-12}
$s$ & $\mathrm{lr}(S)$ & $\mathrm{lr}(\Sigma_i)$ & 1 & 4 & 10 &
$s$ & $\mathrm{lr}(S)$ & $\mathrm{lr}(\Sigma_i)$ & 1 & 4 & 10 \\
\midrule

$0.0$ & 0 & 1e-04 & 19.24 & 56.00* & 63.42* & $0.01$ & 1e-04 & 0 & 30.87 & 57.35* & 57.58\\
$0.0$ & 0 & 5e-04 & 23.35 & 58.49* & 60.53 & $0.01$ & 5e-05 & 0 & 29.39 & 56.79* & 57.98\\
$0.0$ & 0 & 1e-03 & 29.40 & 57.95* & 59.04 & $0.01$ & 5e-04 & 0 & 27.96 & 48.70 & 46.45\\
$0.0$ & 0 & 5e-03 & 32.39 & 49.39* & 51.13* & $0.01$ & 1e-05 & 0 & 18.62 & 55.39* & 60.42\\
 - & - & - & - & - & - & $0.02$ & 1e-04 & 0 & 33.63 & 57.81* & 58.39\\
 - & - & - & - & - & - & $0.02$ & 5e-04 & 0 & 22.73 & 44.30* & 44.66\\
 - & - & - & - & - & - & $0.02$ & 5e-05 & 0 & 28.74 & 57.36* & 57.80\\
 - & - & - & - & - & - & $0.02$ & 1e-05 & 0 & 18.10 & 55.38* & 60.47\\
 - & - & - & - & - & - & $0.1$ & 1e-04 & 0 & 30.68 & 56.74* & 57.15\\
 - & - & - & - & - & - & $0.1$ & 5e-04 & 0 & 16.09 & 40.66* & 41.35\\
 - & - & - & - & - & - & $0.1$ & 1e-05 & 0 & 16.70 & 55.25* & 60.51\\
 - & - & - & - & - & - & $0.1$ & 5e-05 & 0 & 29.67 & 56.46* & 58.21\\
\midrule
$0.1$ & 5e-04 & 5e-03 & 11.16 & 30.43* & 30.18 & $0.1$ & 5e-04 & 1e-03 & 32.96 & 49.11 & 48.18\\
$0.1$ & 5e-04 & 5e-04 & 31.03 & 49.17 & 48.17 & $0.1$ & 5e-04 & 1e-04 & 25.48 & 46.40* & 46.24\\
$0.1$ & 1e-04 & 5e-03 & 27.31 & 48.31* & 50.51 & $0.1$ & 1e-04 & 1e-03 & 41.51 & 57.82* & 59.62\\
$0.1$ & 1e-04 & 5e-04 & 35.25 & 56.78* & 59.59* & $0.1$ & 1e-04 & 1e-04 & 31.77 & 57.24* & 58.24*\\
$0.1$ & 5e-05 & 5e-03 & 34.47 & 50.27* & 51.30 & $0.1$ & 5e-05 & 1e-03 & 35.06 & 59.85* & 60.65\\
$0.1$ & 5e-05 & 5e-04 & 30.20 & 57.82* & 58.28 & $0.1$ & 5e-05 & 1e-04 & 31.45 & 57.83* & 58.25\\
$0.1$ & 1e-05 & 5e-03 & 38.99 & 53.01* & 54.78* & $0.1$ & 1e-05 & 1e-03 & 32.31 & 58.61* & 59.38\\
$0.1$ & 1e-05 & 5e-04 & 29.49 & 58.26* & 59.11 & $0.1$ & 1e-05 & 1e-04 & 20.11 & 57.34* & 60.68\\
$0.02$ & 5e-04 & 5e-03 & 12.97 & 35.81* & 36.48 & $0.02$ & 5e-04 & 1e-03 & 26.76 & 46.17* & 46.72\\
$0.02$ & 5e-04 & 5e-04 & 17.75 & 43.59* & 46.01 & $0.02$ & 5e-04 & 1e-04 & 17.52 & 42.42* & 44.50\\
$0.02$ & 1e-04 & 5e-03 & 35.68 & 51.85 & 53.09 & $0.02$ & 1e-04 & 1e-03 & 38.46 & 58.90* & 60.17\\
$0.02$ & 1e-04 & 5e-04 & 37.96 & 57.75* & 59.53 & $0.02$ & 1e-04 & 1e-04 & 36.83 & 58.43* & 59.43\\
$0.02$ & 5e-05 & 5e-03 & 37.13 & 53.40* & 54.40 & $0.02$ & 5e-05 & 1e-03 & 36.93 & 60.29 & 61.37\\
$0.02$ & 5e-05 & 5e-04 & 34.93 & 59.45* & 60.48 & $0.02$ & 5e-05 & 1e-04 & 31.00 & 57.09* & 57.70\\
$0.02$ & 1e-05 & 5e-03 & 35.77 & 52.01* & 53.59 & $0.02$ & 1e-05 & 1e-03 & 32.72 & 58.32* & 58.35\\
$0.02$ & 1e-05 & 5e-04 & 25.96 & 58.51* & 59.60 & $0.02$ & 1e-05 & 1e-04 & 20.41 & 57.09* & 61.06\\
$0.01$ & 5e-04 & 5e-03 & 15.76 & 38.99* & 39.17 & $0.01$ & 5e-04 & 1e-03 & 25.37 & 46.63* & 46.88\\
$0.01$ & 5e-04 & 5e-04 & 14.54 & 40.91* & 41.32 & $0.01$ & 5e-04 & 1e-04 & 18.47 & 42.72* & 42.71\\
$0.01$ & 1e-04 & 5e-03 & 32.17 & 49.61* & 51.51 & $0.01$ & 1e-04 & 1e-03 & 40.24 & 57.61 & 58.87\\
$0.01$ & 1e-04 & 5e-04 & 35.43 & 58.26* & 59.53 & $0.01$ & 1e-04 & 1e-04 & 29.76 & 55.57* & 55.67\\
$0.01$ & 5e-05 & 5e-03 & 35.53 & 52.59* & 53.45 & $0.01$ & 5e-05 & 1e-03 & 36.95 & 60.55* & 61.69\\
$0.01$ & 5e-05 & 5e-04 & 33.53 & 59.48* & 60.53 & $0.01$ & 5e-05 & 1e-04 & 26.99 & 56.96* & 57.31\\
$0.01$ & 1e-05 & 5e-03 & 33.55 & 48.84 & 51.83 & $0.01$ & 1e-05 & 1e-03 & 31.18 & 60.08* & 61.24\\
$0.01$ & 1e-05 & 5e-04 & 25.73 & 58.55* & 60.45 & $0.01$ & 1e-05 & 1e-04 & 21.07 & 57.86* & 59.09\\
\bottomrule
\end{tabular}
\caption{Ablation results for sun397 accuracy with convergence marks.}
\label{tab:ablation_sun397_accuracy}
\end{table*}

\subsection* {E.1. Full ablation average results}

In the table 4, we attempt to quantify the strength of injected rotational degrees of freedom using the metric 
\begin{equation}
    \tan(\alpha) = \frac{\sqrt{s \cdot (lr_S)^2 \cdot n}}{\sqrt{(lr_{\Sigma})^2}},
\end{equation}

and investigate its relationship with final model performance. The results show that when $\tan(\alpha)$ falls within the range of $[0.05, 0.2]$, the model often achieves superior average accuracy. For instance, under sparsity settings of $s = 0.01$ and $0.1$, configurations with $lr_S = 1\text{e-}4$ and $lr_{\Sigma} = 1\text{e-}3$ consistently yield the best performance across multiple tasks, confirming the positive effect of moderate rotational freedom in enhancing model expressiveness. This observation suggests that carefully tuning the ratio between off-axis and on-axis learning rates to maintain an appropriate rotation angle can effectively improve model performance without significantly increasing the number of trainable parameters.

However, it is worth noting that even when certain configurations satisfy the “optimal” $\tan(\alpha)$ range, their actual performance can still degrade considerably. In particular, when $lr_S$ is set to a large value (e.g., $5\text{e-}4$), the model exhibits a significant drop in average accuracy despite having a theoretically favorable $\tan(\alpha)$. We hypothesize that this is due to the initialization asymmetry: while $\Sigma_i$ is initialized with non-zero values in the range of $(0,1]$ and inherently possesses representational strength, $S$ is initialized to zeros. A large learning rate for $S$ can quickly disrupt the directional structure established by $\Sigma_i$ during early training, undermining the anisotropic geometry of the loss surface and leading to optimization instability.

This phenomenon indicates that while $\tan(\alpha)$ serves as a useful geometric indicator of rotational strength, its effectiveness is conditioned on the interaction between initialization schemes and learning rate assignments across different components. Therefore, relying solely on this angular metric is insufficient for reliably controlling model performance—stable and efficient expressiveness enhancement requires coordinated consideration of initialization strategy, gradient energy distribution, and learning rate design.

\subsection* {E.2. Research on Representative Datasets}

We select four representative datasets from Table 2 in the main results section of the paper. We observe poor performance on \textbf{DTD} and \textbf{RESISC45} compared to GOAT, which achieves state-of-the-art results on both datasets. In contrast, our method achieves the best performance on the \textbf{Cars} dataset, even surpassing Full Fine-Tuning. For \textbf{SUN397}, higher sparsity levels yield better accuracy. We therefore present detailed results for these four datasets to facilitate in-depth analysis.

To support this analysis, Tables 5 through 8 report the complete ablation results on the four selected datasets. For each configuration, we report the best accuracy obtained during different training phases: the “epoch = 4” column reports the best accuracy from epochs 2 to 4, and the “epoch = 10” column reports the best accuracy from epochs 5 to 10. To indicate whether the peak performance occurs at the end of the corresponding interval, we mark values with an asterisk “*” if they represent the maximum value within that window—specifically, an asterisk in the “epoch = 4” column indicates that epoch 4 yields the highest accuracy, and likewise for epoch 10. This convention serves as an approximate indicator of convergence dynamics, highlighting whether further training beyond early epochs yields substantial gains.

\noindent\textbf{DTD Analysis} \
As the smallest dataset in our benchmark suite, \textbf{DTD} only provides 3.76k training images, which poses a significant challenge for convergence within limited training steps. Despite the efficiency of our method, full convergence within just 4 epochs is rarely achieved on DTD. As shown in Table 5, only configurations with a relatively large learning rate for $\Sigma_i$ exhibit signs of early convergence at epoch 4. However, this rapid convergence often comes at the cost of instability, leading to inconsistent final performance and poor generalization.

Notably, when the sparsity level is set to $s = 0.01$, and the learning rates are configured as $\text{lr}(S) = 1 \times 10^{-4}$ and $\text{lr}(\Sigma_i) = 1 \times 10^{-3}$, our method achieves a peak accuracy of 77.17\% in under 10 epochs—surpassing the previous state-of-the-art result of 75.32\% obtained after 76 epochs. This demonstrates the potential of our approach to deliver strong performance with substantially fewer training steps, provided that the hyperparameter configuration is well-calibrated.

\noindent\textbf{RESISC45 Analysis} \
In contrast to DTD, \textbf{RESISC45} represents the other end of the spectrum—a dataset that is notably easy to optimize. As observed in Table 6, a wide range of hyperparameter configurations successfully converge within just 4 epochs, often followed by mild oscillations in later stages of training. This early saturation suggests that the optimization landscape for RESISC45 is relatively smooth and the task less complex.

Under such conditions, granting the model larger parameter freedom while applying smaller learning rates yields better convergence behavior. This setup allows the model to fully leverage its expressive capacity without overshooting optimal regions. Empirically, when the sparsity level is set to $s = 0.1$ with learning rates $\text{lr}(S) = 5 \times 10^{-5}$ and $\text{lr}(\Sigma_i) = 5 \times 10^{-4}$, our method achieves a final accuracy of 93.46\%, matching the state-of-the-art performance. This demonstrates that high-capacity, low-noise updates are crucial in tasks with favorable optimization geometry.

\noindent\textbf{Cars and SUN397 Analysis} \
The \textbf{Cars} and \textbf{SUN397} datasets represent two large-scale benchmarks with dense label granularity—Cars includes 196 vehicle categories, while SUN397 spans 397 scene types. Despite differing in domain, both datasets share a critical characteristic: they offer rich visual diversity and fine-grained supervision, which is ideal for parameter-efficient fine-tuning methods with strong expressive capacity. In such contexts, our method demonstrates superior performance by effectively leveraging the shared basis decomposition while introducing flexible yet minimal updates.

As shown in Table7 and Table8, the best configurations typically correspond to settings with moderate or low sparsity ($s \leq 0.04$) and well-calibrated learning rates. For instance, on Cars, we achieve a top-1 accuracy of \textbf{70.30\%} with $s = 0.04$, $\text{lr}(S) = 5\times10^{-4}$, and $\text{lr}(\Sigma_i) = 5\times10^{-4}$. On SUN397, a setting with $s = 0.01$, $\text{lr}(S) = 1\times10^{-5}$, and $\text{lr}(\Sigma_i) = 1\times10^{-3}$ yields a strong result of \textbf{60.45\%}, even under limited training epochs. This consistency indicates that our method’s rotational update strategy is especially suitable for tasks requiring nuanced discrimination. Rather than relying on aggressive reparameterization, small but structured rotations in the update space suffice to adapt the model effectively—preserving the original representational geometry while enhancing task-specific performance.

\section* {F. Related Work}

\textbf{Parameter-efficient fine-tuning (PEFT) methods} adapt large-scale pre-trained models to specific downstream tasks by modifying only a small fraction of their parameters, thereby significantly reducing computational and memory costs. These methods are particularly valuable in resource-constrained environments (e.g., edge devices) or in scenarios demanding rapid task adaptation or deployment across diverse domains. Traditional PEFT approaches include additive methods \cite{hu_llm-adapters_2023}, selective fine-tuning \cite{guo_parameter-efficient_2021}, and reparameterization-based techniques \cite{aghajanyan_intrinsic_2020, hu_lora_2021}. Among them, Low-Rank Adaptation (LoRA) \cite{hu_lora_2021} efficiently fine-tunes models by inserting low-rank matrices into existing weights, while incurring negligible inference overhead. Extensions of LoRA aim to overcome its inherent limitations and introduce novel optimization schemes, further improving both efficiency and effectiveness. These methods can be categorized into three principal categories: SVD-based LoRA methods, shared-random-basis PEFT methods, and structured decomposition methods.

\noindent\textbf{SVD-based LoRA} methods leverage singular value decomposition (SVD) to refine the initialization and structure of low-rank modules, thereby enhancing convergence speed and adaptability to diverse tasks. SVD decomposes a weight matrix $W$ into $W = U S V^\top$, where U and V are orthogonal matrices, and S is a diagonal matrix containing singular values. Retaining the top r singular components yields a low-rank approximation, significantly reducing parameter counts. AdaLoRA \cite{zhang_adalora_2023} employs SVD for incremental updates by parameterizing weights as $W = W_0 + P \Lambda Q$, where P and Q are singular vectors and $\Lambda$ contains singular values, enabling dynamic rank adaptation to improve adaptability and training efficiency. LoftQ \cite{li_loftq_2023} integrates quantization and low-rank approximation by expressing weights as a combination of a quantized matrix Q and low-rank terms $AB^\top$, optimizing the approximation under the Frobenius norm via alternating optimization. LQ-LoRA \cite{guo_lq-lora_2023} similarly separates weights into quantized and low-rank components, updating only the latter to minimize training overhead. SVDiff \cite{han_svdiff_2023} tunes singular values in diffusion models, achieving a ~2200× parameter reduction over DreamBooth and proving particularly effective for generative modeling. PiSSA \cite{meng_pissa_2024} accelerates convergence using progressively sparse structures initialized with top singular vectors/values, freezing the remainder—a strategy well-suited for sparse tasks. SC-LoRA \cite{luo_sc-lora_2025} constrains LoRA within SVD-derived subspaces, striking a trade-off between parameter efficiency and knowledge preservation. LoRA-GA \cite{wang_lora-ga_2024} leverages gradient SVD for optimized adaptation. While these methods offer enhanced initialization or decomposition strategies, storing full singular vector matrices incurs non-trivial memory overhead, prompting most approaches to preserve only the top r directions—potentially limiting performance on complex tasks \cite{sun_singular_2022, sun_svfit_2024, lingam_svft_2024}.

\noindent\textbf{Shared-random-basis PEFT} methods achieve parameter compression by sharing random subspaces across layers, tasks, or even modalities, thereby offering strong storage efficiency, though often at the cost of task-adaptive expressiveness. VeRA \cite{kopiczko_vera_2023} first utilizes a shared low-rank matrix pair across all layers and learns lightweight per-layer scaling vectors, achieving a 10× parameter reduction compared to LoRA while maintaining competitive performance on GLUE and E2E benchmarks. NOLA \cite{koohpayegani_nola_2023} reparameterizes LoRA by expressing the low-rank matrices as linear combinations of fixed random basis matrices, learning only the mixing coefficients during training. This approach reduces memory usage by nearly 20× on LLaMA-2 70B while preserving accuracy. VB-LoRA \cite{li_vb-lora_2024} introduces variational Bayesian modeling to capture uncertainty in low-rank adaptation, enhancing robustness and generalization under limited supervision. RandLoRA \cite{albert_randlora_2025} investigates randomized full-rank matrices to approximate low-rank updates, striking a balance between expressive capacity and parameter efficiency. These methods are particularly suited for resource-constrained deployment environments, enabling the storage of multiple model variants with minimal overhead. However, the reliance on shared randomness can limit task-adaptive expressiveness, potentially impacting performance on highly specialized tasks \cite{li_vb-lora_2024, albert_randlora_2025}.

\noindent\textbf{Structured decomposition methods} factorize weight matrices into separable magnitude and directional components, enabling more granular updates and closely approximating full fine-tuning capabilities. DoRA \cite{liu_dora_2024} decomposes the pre-trained weight $W_0$ into a magnitude term m = $\|W_0\|_c$ and a direction term $V = W_0 / \|W_0\|_c$, updating only the direction via a low-rank approximation $\Delta W = BA$, where $B$ and $A$ are low-rank matrices. This decoupled control over magnitude and direction emulates full fine-tuning behavior, significantly boosting learning capacity and optimization stability without incurring additional inference overhead. Dynamic DoRA \cite{mao_dora_2024} further enhances efficiency by pruning unimportant updates during training, yielding notable performance gains on MRPC, STS-B, and SST-2. DoRA’s structured decomposition achieves a strong balance between parameter efficiency and learning effectiveness, making it a versatile and effective PEFT framework. Our work draws inspiration from DoRA’s decomposition strategy to enhance learning capacity, and incorporates SVD-based initialization techniques to further enhance LoRA’s adaptability and parameter efficiency.

\section* {G. Limitations and Future Work}
\subsection*{G.1. Limitations}
\noindent\textbf{Our theoretical analysis is grounded in a vector-based abstraction of the underlying tensor structure.} Specifically, we flatten the second-order tensor updates into first-order vector representations in order to simplify the characterization of rotational behavior. Within this abstraction, the rotation angle $\alpha$ is defined based on the Frobenius norm ratio between two orthogonal components—namely, the on-axis (aligned with $\Sigma_i$) and off-axis (captured by $S$) directions. This yields a scalar approximation of the overall rotation degree, which facilitates empirical analysis across datasets and hyperparameters.

However, this scalar $\alpha$ inevitably oversimplifies the rich geometry of tensor rotations. A single angle cannot fully describe the distributed rotational behavior of a high-dimensional second-order tensor. Furthermore, our current formulation aggregates all singular components uniformly through the Frobenius norm, without distinguishing between dominant and weak directions. Intuitively, these weaker components may undergo disproportionately larger angular corrections during fine-tuning, yet their contributions are averaged out in the current scalar estimate. A more granular analysis—e.g., by computing per-component angular deviation or introducing a weighted norm formulation—may reveal finer insights into how our method selectively adapts under different sparsity and strength regimes. This limitation opens up a promising direction for future work on disentangled, component-wise tensor rotation modeling.

\noindent\textbf{Computational Overhead from Full-Rank Structured Multiplication.}
Although our method achieves full-rank expressivity and efficient memory usage through sparse updates, it introduces substantial computational cost during both training and initialization. Specifically, our update takes the form $U (\Sigma + S) V^\top$, which involves the multiplication of three matrices with shape $(m \times n)$, $(n \times n)$, and $(n \times n)$, respectively. The resulting computational complexity is $\mathcal{O}(mn^2)$, which is significantly higher than LoRA’s low-rank factorization with complexity $\mathcal{O}(mr + nr)$ when $r \ll n$.

Moreover, while the matrix $S$ is sparse, it still resides in a large $n \times n$ space, and the sparsity alone does not eliminate the overhead from the nonzero elements, especially when $n$ is large. This leads to noticeably longer training times despite the memory advantages. Additionally, our method requires an initialization phase where all task-specific matrices are jointly decomposed through a concatenated SVD or GSVD operation, which adds a one-time but nontrivial computational burden. Although this step is only performed once per model, it becomes a bottleneck for large-scale or continual learning scenarios.

\noindent\textbf{We leave the evaluation of our method on models larger than 70B for future work, due to current resource limitations.} At present, our computational infrastructure—both in terms of memory capacity and parallel processing capabilities—is not sufficient to support full-scale training or inference on such large models. Running experiments at this scale typically requires specialized hardware (e.g., high-bandwidth interconnects, multi-node GPU clusters) and significant engineering overhead, which are beyond the scope of our current environment. Nevertheless, we believe our method is conceptually scalable, and we plan to conduct large-scale validation when the required resources become available.

\subsection*{G.2. Future Works}
\noindent\textbf{Refinement of the tensor basis theory}

Our current theoretical framework simplifies second-order tensor updates into first-order vector abstractions, which enables a tractable definition of the rotation angle $\alpha$ as a scalar derived from the ratio of Frobenius norms between orthogonal update components. While this abstraction offers intuitive insights and supports empirical analysis, it falls short in capturing the full geometric complexity of high-dimensional tensor rotations. In particular, a single scalar angle cannot faithfully represent the distributed rotation across different singular directions within the matrix.

In future work, we plan to extend this theory by explicitly modeling tensor rotations at a more fine-grained level. First, we will investigate per-component angular deviations, distinguishing between dominant and weak singular directions. This will allow us to characterize how rotational freedom differentially affects strong versus weak basis components—an aspect currently obscured by the uniform Frobenius norm aggregation. Second, we aim to construct a more expressive representation of the rotation field, potentially in the form of a structured angular matrix or spectral rotation tensor, which can capture direction-specific rotational dynamics. Lastly, we intend to analyze how these refined formulations interact with optimization behavior under various sparsity constraints, ultimately leading to a more principled and component-aware design of rotational PEFT methods.

\noindent\textbf{Efficient Sparsity Methods}

To mitigate the computational burden associated with full-rank structured updates, future research should investigate structured sparsity strategies that retain expressive capacity while improving computational efficiency. One promising direction is to constrain the sparse matrix $S \in \mathbb{R}^{n \times n}$ to exhibit structured patterns such as block-diagonal or banded forms, which can leverage optimized matrix routines and hardware acceleration for faster multiplication. By carefully designing the support of S, the cost of the $(\Sigma + S)V^\top$ product can be reduced from $\mathcal{O}(n^2)$ to $\mathcal{O}(ns)$, where s denotes the number of non-zero entries per row.

Another complementary direction is to decouple the sparsity of S across parameter groups, enabling each module to operate within a localized subspace of the full n-dimensional parameter space. This approach not only reduces computational complexity but also enhances parallelism during training.

For the expensive one-time initialization involving GSVD across all task-specific matrices, approximate or incremental GSVD schemes could be explored to reduce preprocessing cost. For instance, randomized SVD or block-wise GSVD may offer acceptable accuracy at significantly lower computational cost, thereby improving scalability to large-scale models and continual learning scenarios. Additionally, caching decomposition statistics across similar tasks may further amortize this cost in multi-task learning settings.

\noindent\textbf{Continual Learning and Model Merging} 

LoRI\cite{zhang_lori_2025} mitigates cross-task interference and catastrophic forgetting by enforcing high sparsity in its task-specific updates, formulated as $\Delta_t = A(B_t \odot M_t)$. By freezing the projection matrix $A$ and selectively updating $B_t$ using sparse masks $M_t$, LoRI achieves task-specific isolation at low memory overhead. However, this design inherently limits model capacity and representational degrees of freedom, particularly under complex or non-stationary task distributions.

In contrast, our proposed method, FRoD, introduces a full-rank, rotation-enhanced adaptation space via sparse perturbations $S_i$, decomposing updates into axis-aligned magnitudes and orthogonal rotational refinements. This structured mechanism preserves the generalized singular spectrum while providing flexibility for task-dependent adaptation, making it especially suitable for continual learning with enhanced expressivity and reduced forgetting.

Moreover, while LoRI enables adapter merging via random subspace alignment, the resulting representations remain constrained within sparse, low-rank manifolds. FRoD, by contrast, learns a globally shared latent basis $V$ and applies sparse rotation-based perturbations $S_i$ atop this shared basis, enabling geometrically meaningful and high-capacity model composition. By explicitly modeling the rotation angle $\alpha$ and preserving spectral coherence throughout model fusion, FRoD enables consistent integration across heterogeneous tasks, achieving superior continual adaptation and compositional generalization relative to LoRI.

\noindent\textbf{Federated Learning}

FRoD exhibits several properties that naturally align with the constraints and goals of federated learning. First, the structured decomposition of task updates into on-axis magnitude and off-axis rotational components promotes linearity and modularity in the parameter space. This factorization not only simplifies update representation but also facilitates lightweight synchronization across clients, making it compatible with gradient compression and communication-efficient protocols.

Second, FRoD enables rapid convergence by preserving spectral stability and injecting minimal yet expressive rotations, which effectively bridge the gap between low-rank initialization and full-rank expressivity. Such convergence behavior is particularly valuable in federated settings where clients are often limited to a few local epochs per round due to latency or privacy constraints.

Third, while traditional full-rank fine-tuning incurs prohibitive memory and communication costs in federated environments, FRoD achieves high-rank adaptation through sparse perturbations layered over shared latent bases. This design supports scalable and parameter-efficient tuning with minimal transmission overhead, as only the sparse update $S_i$ needs to be communicated per task or client.

Taken together, these properties position FRoD as a promising solution for federated learning in resource-constrained scenarios, offering a favorable trade-off between expressivity, communication cost, and convergence speed. We leave a full evaluation of FRoD under federated training regimes as future work.

\putbib[reference]
\end{bibunit}
\end{document}